\documentclass[twoside]{article}

\usepackage[accepted]{aistats2022}
\usepackage[round]{natbib}

\usepackage[T1]{fontenc}    
\usepackage{lmodern}
\usepackage{hyperref}       
\usepackage{url}            
\usepackage{booktabs}       
\usepackage{microtype}      
\usepackage{graphicx}
\graphicspath{{figs/}}
\usepackage{subcaption}
\usepackage{placeins}
\usepackage{hyperref}       
\usepackage[dvipsnames]{xcolor}

\usepackage{array}
\usepackage{wrapfig}

\hypersetup{ 
	pdftitle={MAP convergence},
	pdfkeywords={},
	pdfborder=0 0 0,
	pdfpagemode=UseNone,
	colorlinks=true,
	linkcolor=blue, 
	citecolor=blue, 
	filecolor=blue, 
	urlcolor=blue, 
	pdfview=FitH,
	pdfauthor={Anonymous},
	final,
}

\usepackage[toc,page,header]{appendix}
\usepackage{minitoc}

\usepackage[framemethod=TikZ]{mdframed}
\newmdenv[
	roundcorner=10pt,
	backgroundcolor=Red!30,
]{important}

\usepackage[capitalise]{cleveref}

\usepackage{math_commands}
\newtheorem{proposition}{Proposition}
\newtheorem{problem}{Open Problem}

\newcommand*{\expect}[2][]{\ensuremath{\mathbb{E}_{#1} \left[ #2 \right] }} 
\newcommand*{\expecti}[2][]{\ensuremath{\mathbb{E}_{#1} [ #2 ] }} 

\newcommand{\cond}{\,\vert\,}
\newcommand{\logpart}{A}
\newcommand{\conj}{{\logpart^*}}
\newcommand{\bregman}{\cB_\logpart}
\newcommand{\bregmanconj}{\cB_{\logpart^*}}
\newcommand{\nat}{\theta}
\newcommand{\m}{\mu}
\newcommand{\meanp}{\m}

\newcommand{\lr}{\gamma} 
\newcommand{\lin}[1]{\left\langle#1\right\rangle}

\newcommand{\MAPm}{\hat \m_n}
\newcommand{\MAPt}{\hat \nat_n}
\DeclareMathSymbol{\shortminus}{\mathbin}{AMSa}{"39}

\newcommand{\stgcvx}{\alpha} 
\newcommand{\smooth}{\beta} 

\begin{document}

\runningtitle{Convergence Rates for the MAP of an Exponential Family and SMD -- an Open Problem}

\runningauthor{R\'emi Le Priol, Frederik Kunstner, Damien Scieur, and Simon Lacoste-Julien}

\twocolumn[

\aistatstitle{Convergence Rates for the MAP of an Exponential Family\\ and Stochastic Mirror Descent -- an Open Problem}

\aistatsauthor{R\'emi Le Priol \And Frederik Kunstner}
\aistatsaddress{ 
	Mila, Université de Montreal 
	\And  University of British Columbia
	}
	
\aistatsauthor{Damien Scieur \And Simon Lacoste-Julien}
\aistatsaddress{Samsung, SAIT AI Lab, Montreal 
	\And Mila, Université de Montreal \\
	Samsung, SAIT AI Lab, Montreal\\
	Canada CIFAR AI Chair
}

]

\begin{abstract}
We consider the problem of upper bounding the expected log-likelihood sub-optimality of the maximum likelihood estimate (MLE), or a conjugate maximum a posteriori (MAP) for an exponential family, in a non-asymptotic way.
Surprisingly, we found no general solution to this problem in the literature. In particular, current theories do not hold for a Gaussian or in the interesting few samples regime.
After exhibiting various facets of the problem, we show we can interpret the MAP as running stochastic mirror descent (SMD) on the log-likelihood. However, modern convergence results do not apply for standard examples of the exponential family, highlighting holes in the convergence literature.
We believe solving this very fundamental problem may bring progress to both the statistics and optimization communities.
\end{abstract}

\doparttoc 
\faketableofcontents 

\begin{figure}[t]
	\centering
\includegraphics[width=.4\textwidth]{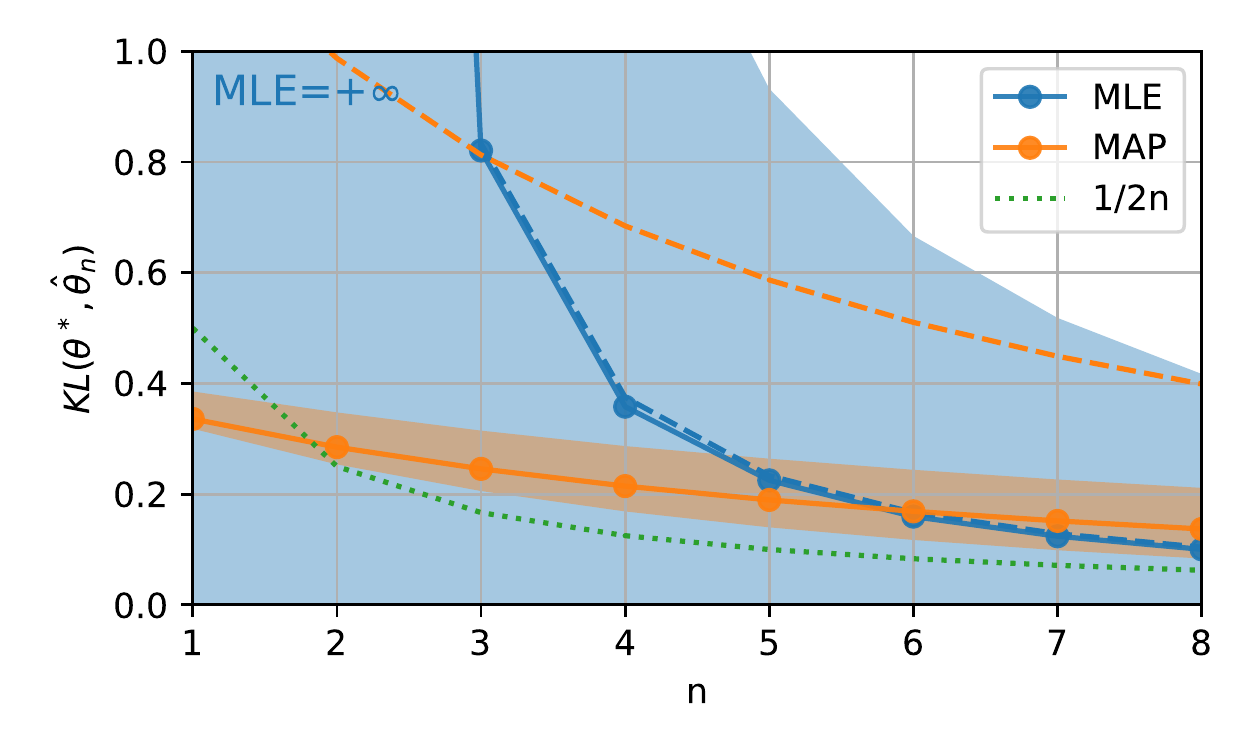}
	\caption{KL divergence~\eqref{eq:suboptimalityKL} for Gaussian variance (\S\ref{ssec:gaussian-variance}) MLE (blue) and MAP (orange) against number of samples $n$. 
		Solid curve  are average over $10^5$ trials.
		Dashed curves are upper bounds~\eqref{eq:MLE_rate} (blue) and~\eqref{eq:MAP_rate} (orange, not tight by a factor 2).
		Shaded areas are 90\% confidence interval.
		The MLE expected KL is infinite for $n=1$ and $n=2$, but for $n\geq3$ it quickly joins the upper bound~\eqref{eq:MLE_rate} and the $1/2n$ asymptote~\eqref{eq:asymptote}.
		MAP's expected KL is always finite, and it has lower variance than MLE, but it is slower to join the asymptote.
		We wish to find upper bounds similar to~\eqref{eq:MAP_rate} characterizing the relative importance of the prior and the few sample behavior of MAP for a variety of exponential families.
	}
	\label{fig:curves}
\end{figure}

\section{INTRODUCTION}
\label{sec:motivation}

\paragraph{Models}
Exponential families are among the most widely used simple parametric models of data, yet, we will highlight some open problems about them in this paper.
Many standard random variables are exponential families: Gaussians, categorical, gamma, or Dirichlet, for example.
They are flexible enough to model a variety of data sources $X$ and easy to describe with some sufficient statistics $T(X) \in \real^d$.
They are particularly appreciated for their convex log-likelihood
\alignn{
f(\nat) := \E[-\log p_\nat(X)] = \logpart(\nat) - \lin{\E[T(X)] , \nat},
\label{eq:defNLL}
}
where $\logpart$ is the convex log-partition function and \mbox{$\nat\in\Theta$} is the \textit{natural parameter}.
This convexity lays the foundation for generalized linear models \citep{mccullagh1989generalized} or variants of principal component analysis \citep{collins2001generalization}, among other applications.

\paragraph{Estimators}
In this paper, we consider the problem of estimating $\nat$ from a dataset $\mD = (X_1, \dots, X_n)$ of iid observations from $p_\theta$ in an exponential family.
In this case, not only is $f$ convex, but it yields a simple condition for the maximum-likelihood estimate (MLE)
\begin{align}
 \hat \mu_n^\text{MLE} = \nabla  \logpart(\hat \nat_n^\text{MLE}) = \frac{\sum_{i=1}^n T(x_i)}{n} \; .
	\label{eq:defMLE}
\end{align}
This rule is also known as moment matching.
Given a specific conjugate prior, a similar formula~\eqref{eq:defMAP} holds for the maximum a posteriori (MAP). In this paper, we will focus on analyzing MLE and MAP estimators.\footnote{A related analysis is present in the online-learning literature, but for different online estimators, which are less efficient than offline methods \citep{azoury2001relative,dasgupta2007online}.}

\paragraph{Statistical decision theory}
To assess the quality of an estimator $\hat \nat$ (and compare them), we need to define some notion of closeness to the correct parameter $\nat^*$.
We distinguish here two ways: \textit{distance in parameter space} and \textit{``distance'' between distributions}.
{\bf 1)} Distance in parameter space $d(\nat,\nat^*)$. This is the focus of \emph{parameter estimation}, yielding results such as the asymptotic efficiency of the MLE via the Cramer-Rao lower-bound \citep{aitken1942estimation} and a wealth of asymptotic results \citep{vdv1998asymptotic}.
In particular for sum of independent variables such as~\eqref{eq:defMLE}, large deviations theory \citep{varadhan1984large} characterizes concentration phenomena.
{\bf 2)} Distance between distributions, as studied in \emph{density estimation}.
For this purpose, the Kullback-Leibler (KL) divergence $\KL\paren{p_{\nat^*} || p_{\nat} }$  arises naturally from information theory,
but its lack of robustness to misspecification\footnote{
For $p$ and $q$ continuous densities,
$\KL(p||q) = +\infty$ if $\exists x, q(x)=0 \, \& \, p(x)>0$.
}
has led statisticians to study symmetric, better-behaved distances, such as the $L^2$ norm \citep[\S1.2]{tsybakov2009introduction}, the $L^1$ norm \citep{devroye2001combinatorial} or more recently the Hellinger distance \citep{baraud2017new}.
With exponential families, the KL divergence is also a Bregman divergence between parameters (see \S\ref{sec:problem}), thus drawing a connection between these two lines of research,
and raising the fundamental problem:
\begin{equation}
\boxed{
\begin{aligned}
	\textit{Find an upper}&\textit{ bound on the expected value of } \\
	&\KL(p_{\nat^*} || p_{\hat \nat_n^\text{MLE/MAP} }) \; .
\end{aligned}
}
\tag{$\star$}
\label{problem}
\end{equation}

There are already general asymptotic results (\S\ref{ssec:asymptote} and Fig.~\ref{fig:curves}), and a finite $n$ result when $\logpart$ is quadratic (e.g., $X$ is Gaussian with known variance)
or close to quadratic (\S\ref{ssec:quadratic}).
However, a  general solution for finite $n$ remains elusive. 
In this paper, we review partial solutions and give ideas on how to solve the problem.

\paragraph{Optimization} 
Stochastic optimization offers an interesting perspective on~\eqref{problem}.
Consider the problem
\alignn{
	\min_{\theta\in \Theta} f(\theta)\,,
	\label{eq:optimization_problem}
}
solved by $\nat^*\in \Theta$.
Setting $f$ as the log-likelihood~\eqref{eq:defNLL}, the suboptimality is equal to the KL:
\alignn{
	f(\nat) - f(\nat^*) = \KL\paren{p_{\nat^*} || p_{\nat} } .
	\label{eq:suboptimalityKL}
}
Both MLE and MAP can be seen as stochastic algorithms solving~\eqref{eq:optimization_problem}.
In particular, with exponential families, MAP is equivalent to stochastic mirror descent (SMD) \citep{nemirovski2009robust}.
Inspired by recent work~\citep{lepriol2021analysis, kunstner2020homeomorphic}, we consider using existing convergence rates for SMD to get the upper bound we seek.
Unfortunately, none of the current analyses apply, highlighting open problems for the analysis of SMD.

\paragraph{Expected Outcomes}
A solution to~\eqref{problem} can clarify the importance of the prior in MAP, in particular in the few sample regime. 
Also, it could enable stochastic optimization to tackle a broad class of barrier objectives.\footnote{we call \emph{barrier} an objective $f$ that is infinite on the boundaries of its domain (assuming they exist).}
A good example is the generalized linear model based on Gaussians with unknown mean and variance, for which there is currently no theory \citep{bach2013nonstronglyconvex}.
It could also help assess the impact of alternative forms of regularization (prior) for these models.

\paragraph{Contributions}
After formalizing the problem~\eqref{problem} (\S\ref{sec:problem}), along with its asymptotic properties (\S\ref{ssec:asymptote}), we make the following contributions.
\begin{itemize}
	\itemsep0em
	\item We provide an upper bound on the KL in the particular case of a Gaussian with known mean but unknown variance $\cN(0,\sigma^2)$ (\S\ref{ssec:gaussian-variance}), illustrating that tight rates are possible even though the current theory does not cover them.
	\item We highlight sufficient conditions to characterize when a (local) quadratic approximation of the KL is valid, offering a partial answer to~\eqref{problem} (\S\ref{ssec:quadratic}-\ref{ssec:local-quadratic}).
	\item By linking MAP and SMD, we show that modern analysis of SMD is yet to prove convergence on barrier objectives such as $-\log$ (\S\ref{sec:optimization}).
\end{itemize}

\paragraph{Notation}
$X$ and $T=T(X)$ are random variables, $x$ is a sample, $n$ is the number of samples and $d= \dim(T)$.
$\langle \cdot , \cdot \rangle$ is the Euclidean scalar product in $\real^d$.

\section{TECHNICAL BACKGROUND}
\label{sec:background}
This section reviews the formalism of exponential families, their duality, a conjugate prior, and the corresponding MAP.
We point the reader towards \citet[Chapter 3]{wainwright2008graphical} for a more detailed introduction.

The density of an exponential family for a sample $x$ is
\begin{equation}
	 p_\nat(x) = p(x|\nat) = \exp( \langle \nat, T(x) \rangle - \logpart(\nat)) \; ,
	 \label{eq:def_expfamily}
\end{equation}
where  $\nat$ is called natural (or primal) parameter.
It is fully specified by 1) $T: \cX \rightarrow \real^d$, the sufficient statistic,
and 2) a base measure $\nu$ on $\cX$.
Since the exponential is positive, $p$ has the same support as $\nu$.
The log-partition function $\logpart$ acts as a normalization term, since
\begin{align}
    \logpart(\nat) = \log \int e^{\langle \nat, T(x) \rangle} \nu(dx) \;.
\end{align}
This simple model encompasses both categorical distributions : $\cX = \{1, \dots, k\}$, $\nu$ uniform and $T(X)$  the one-hot encoding and multivariate normal distributions $\cX=\real^d, \nu$ Lebesgue and $T(X)=(X, X X^\top)$.

For convenience, we focus on steep, regular exponential families with minimal statistic $T$ \citep{barndoffnielsen2014information}.
Then $\logpart$ is a strictly convex function of Legendre type,
and the set $\Theta = \{ \nat \cond \logpart(\nat) < \infty\}$ is open and convex.
When explicit, we write the random variable $T = T(X)$.

{\bf Duality.}
The log-partition function $\logpart$ verifies the two following identities:
\begin{align}
    \nabla\logpart(\nat) &=  \expect[p_\nat]{T(X)} =: \meanp, \label{eq:mirror-map} \\
    \nabla^2 \logpart(\nat) &= \Cov_{p_\nat}[T(X)] > 0,
\end{align}
where $\meanp$ is called the mean (or dual) parameter, which lives in the open convex set $\cM$ equal to the relative interior of the convex hull of $T(\cX)$.
Given that $\logpart$ is strictly convex, its Hessian is positive definite, and its gradient $\nabla \logpart$ is a \textit{bijection} between natural parameters $\nat$ and mean parameters $\m$.
We will write $\m$ or $\nat$ interchangeably depending on the context, being aware that both are linked and represent the same distribution.

We now introduce the convex conjugate (the Fenchel-Legendre transform) of the log-partition function
\aligns{
	\conj(\m) =  \langle \m, \nat \rangle - \logpart(\nat)
	=  \max_{\nat'\in\Theta}  \langle \m, \nat' \rangle - \logpart(\nat')\; ,
}
which is the common notion of \textit{entropy} in information theory.
By Fenchel duality, its gradient is the inverse of the gradient of $\logpart$,  $\nabla\conj=\nabla\logpart^{-1}$, giving
\aligns{
	\nabla\conj \circ \nabla\logpart(\nat) = \nat, \quad \nabla\logpart\circ \nabla\conj(\meanp) = \meanp.
}

{\bf The Bregman Divergence} induced by $\logpart$ measures the discrepancy between two parameters $\nat$ and $\nat_0$,
\begin{align}
    \bregman (\nat ; \nat_0)
    & = \logpart(\nat) - \logpart(\nat_0)
    - \langle \nabla \logpart(\nat_0)  , \nat - \nat_0 \rangle,
    \label{eq:defBregman}
\end{align}
with $\nabla \logpart(\nat_0) = \expect[\nat_0]{T(X)} =: \meanp_0$ the mean parameter associated to $\nat_0$.
In general, Bregman divergences are not symmetric, i.e., $\bregman (\nat ; \nat_0)\neq \bregman (\nat_0 ; \nat)$.

{\bf A Conjugate Prior} for $p(X|\nat)$ is
\begin{align}
    p(\nat)
    &\propto \exp( - n_0 \bregman(\nat ; \nat_0) ) \nonumber \\
    &\propto \exp(n_0 \langle \m_0, \nat \rangle - n_0 \logpart(\nat)),
    \label{eq:def_prior}
\end{align}
where $n_0$ and $\nat_0$ are (hyper)parameters of the prior  \citep{agarwal2010geometric}.
This is an exponential family with sufficient statistics $(\nat ,\logpart(\nat))$ and natural parameter $(n_0 \m_0, -n_0)$.
Intuitively, $n_0$ is the number of fictive data points observed from a distribution with natural parameter $\nat_0$.

{\bf Maximum A Posteriori (MAP).}
Given a dataset $\mD_n =(X_1,\dots,X_n)$, we wish to estimate the maximum of the posterior distribution $p(\nat \cond \mD_n) \propto p(\mD_n|\nat)p(\nat)$.
Plugging in~\eqref{eq:def_expfamily},~\eqref{eq:defBregman} and~\eqref{eq:def_prior} yields
\aligns{
	p(\nat \cond \mD_n)
    \propto \exp(- (n_0+n) \bregman(\nat; \MAPt^\text{MAP}))
}
which reaches its maximum at $\MAPt^\text{MAP}$ such that
\begin{align}
    \nabla \logpart(\MAPt^\text{MAP}) = \MAPm^\text{MAP}
    = \frac{n_0 \meanp_0 + \sum_{i=1}^n T_i}{n_0+n} \; ,
    \label{eq:defMAP}
\end{align}
where $T_i=T(X_i)$.
When $n_0=0$ (no samples from the prior), we recover the MLE~\eqref{eq:defMLE}.
We write $\MAPt$ for the MAP and view the MLE as a particular case.

\section{PROBLEMS FORMULATION}
\label{sec:problem}.

We are now ready to formalize the main problem of this paper. Assume we observe a dataset $\mD_n$ drawn i.i.d. from $p(\cdot \cond\nat^*)$, an exponential family distribution
with parameters $\nat^*$.
We wish to quantify how well the MLE or a MAP approximates the true distribution.

A natural way to quantify this is the Kullback-Leibler divergence (KL) $\KL(p_{\nat^*} || p_\nat)$.
In the well-specified setting, it corresponds to the log-likelihood sub-optimality~\eqref{eq:suboptimalityKL}.
With exponential families, the KL is also a Bregman divergence:
\alignn{
	\KL(p_{\nat^*} || p_\nat)
	 = \bregman(\nat ; \nat^*)
	 = \bregmanconj(\m^* ; \m) \; .
}
The second equality is a general property of Bregman divergences and convex conjugates. 
How does this quantity behave when $\hat \nat$ is the MLE or MAP?
 Or in the words of statistical decision theory, what is the \emph{frequentist risk} of these estimators when the loss is the KL divergence?
This is our first problem.

\begin{problem}[Upper-bounding MAP and MLE]
Upper bound the following quantities:
\begin{align}
	\label{eq:bregmanMLE}
	\text{MLE: } \quad &\expect[\mD_n]{\bregmanconj \left (\E_{\nat^*} [T] ;  \inv{n}  \smallsum_i T_i \right )}, \\
	\label{eq:bregmanMAP}
	\text{MAP: } \quad &\expect[\mD_n]{\bregmanconj \left (\E_{\nat^*} [T] ; \frac{n_0 \m_0 + \smallsum_i T_i}{n_0+n} \right )},
\end{align}
where the expectation is on the data $\mD_n = (T_1, \dots, T_n)$.
\end{problem}

More explicitly, we want an upper bound that does not involve this expectation over the dataset.
Surprisingly, we found no general solution to this seemingly simple problem, whether in the literature or by our means.
In \S\ref{sec:example}, we provide results for special cases such as $\cN(0, \sigma^2)$,
while in \S,\ref{sec:insights} we provide realistic conditions to obtain valid bounds after seeing a large number of samples.
However, we have yet to find a solution encompassing both a broad range of exponential families
and applicable to small sample sizes $n \lesssim d$.

{\bf A Difficulty with the MLE.}
While~\eqref{eq:bregmanMAP} is always finite, \eqref{eq:bregmanMLE} may be infinite,
for instance when estimating the covariance of a Gaussian when $n \leq d + 1$.
Even worse, there is a non-zero probability never to sample one of the categories with categorical variables.
In those cases the MLE gives zero weight to this category and $\KL(p_{\nat^*} \,\Vert\, p_\text{MLE} ) = +\infty$.
Therefore, the expected KL~\eqref{eq:bregmanMLE} is infinite for any number of samples.
Instead of taking the expectation, one might want to bound the risk in high probability
without resorting to Markov inequality, as achieved by~\citep{ostrovskii2021finite},
 but this is a difficult endeavor.
These examples make a case for regularized estimators such as MAP,
for which we may find upper bounds.

\paragraph{Optimization.}
With exponential families, MAP can be linked to \emph{stochastic mirror descent (SMD)}, see App.~\ref{app:SMD}.
More precisely, let us re-write~\eqref{eq:defMAP} as
\alignn{
\m_n = \m_{n-1}- \lr_n (\m_{n-1} - T_n)
\label{eq:mean_update}
}
where $\lr_n := \inv{n_0 + n}$.
Now define stochastic functions $f_X(\nat) = -\log p(X \cond \nat)$ such that $\E[f_X] = f$.
If we further introduce stochastic gradients $g_n(\nat) := \nabla\logpart(\nat) - T_n = \nabla f_{X_n}(\nat)$, then~\eqref{eq:mean_update} becomes
\alignn{
	\nabla\conj(\hat \nat_{n})
	= \nabla\conj(\hat \nat_{n-1}) - \lr_n g_n(\hat \nat_{n-1}),
}
which is the update formula for SMD on $f$ with mirror map $\nabla\logpart$
and step-size schedule $\lr_n$, initialized at $\nat_0$.
In this view, MLE forgets its (arbitrary) initialization after the first step with step size 1.
The observation MAP $\in$ SMD brings us to our second problem.
\begin{problem}[Convergence rate for SMD]
Find a convergence rate for stochastic mirror descent that applies to conjugate MAP of exponential families such as Gaussians $\cN(\meanp, \sigma^2)$.
\end{problem}

To address these problems, we start by investigating simple examples to provide solutions to Problem 1, getting insights into what is achievable.

\section{ILLUSTRATING EXAMPLES}\label{sec:example}
\subsection{Gaussian with Unknown Variance}\label{ssec:gaussian-variance}

A non-trivial yet straightforward example is the centered Gaussian distribution with unknown variance $\cN(0,\sigma^2)$.
Its log-likelihood reads $\log p(x) = -\frac{x^2}{2 \sigma^2} - \half\log(2\pi \sigma^2)$.
Defining $T(X)=X^2$ as the sufficient statistic, we get natural parameter $\nat = -\inv{2 \sigma^2} <0$, and mean parameter $\m=\E[T(X)] = \sigma^2 >0$.
Mean and natural parameters are roughly inverse of each other, i.e., $\nat = -\inv{2 \m}$.
Now we match the log-likelihood with the exponential family template to get the log-partition function, and take the conjugate to find the entropy
\aligns{
	\logpart (\nat) = - \half \log(-\nat) \quad\text{and}\quad
	\conj(\m) = - \half \log(\m)  \; ,
}
up to constants.
Both $A$ and  the entropy are roughly negative logarithm $z\mapsto - \log(z)$.
It means the conjugate prior is the exponential family with sufficient statistic $(\nat, \log(-\nat) )$, e.g., a negative gamma distribution.
It also means $\bregman$ and $\bregmanconj$ have the same shape
\begin{align}
	\bregmanconj( \m^*; \m_n)
	&= \half \left ( \frac{\m^*}{ \m_n} - 1 - \log  \frac{\m^*}{ \m_n} \right) \; .
\end{align}
In Theorems~\ref{thm:varianceMLE} and~\ref{thm:varianceMAP}, we report upper bounds on the expected value of this divergence for the MLE and the MAP.
All proofs for this section are in App.~\ref{app:gaussian-variance}.
\begin{theorem}[MLE Bound]
\label{thm:varianceMLE}
	The MLE of $\cN(0,\m^*)$ is $\hat \m_n^\text{MLE} = \inv{n} \sum_i X_i^2 $.
	Its expected suboptimality is infinite when $n\leq 2$, and otherwise upper-bounded as
	\begin{align}
		 \expect{\bregmanconj( \m^*; \hat \m_n^\text{MLE}) }
			\leq \inv{2n} +\frac{2}{n(n-2)} \; .
			\label{eq:MLE_rate}
	\end{align}
\end{theorem}
This upper bound matches the asymptotic result~\eqref{eq:asymptote} that we derive in \S\ref{ssec:asymptote}.
We illustrate its numerical behavior in Figure~\ref{fig:curves}.
With the same technique, we obtain a similar bound for the multivariate generalization: the expected value is infinite whenever $n \leq d+1$ where $d$ is the dimension, and is otherwise bounded by $O\paren{\frac{d^2}{n} + \frac{d^3}{n(n-d-1)}}$.

\begin{theorem}[MAP Bound]
\label{thm:varianceMAP}
The expected suboptimality of the MAP of $\cN(0,\m^*)$ with prior hyper-parameters $(n_0,\m_0)$ is
 \begin{align}
	& \expect{\bregmanconj( \m^*; \hat \m_n^\mathrm{MAP})}
	\leq \begin{cases}
		\inv{2(n_0+1)}  +  b_1 \ \text{if}\ n=1,\\
		\frac{1}{n_0 \frac{\m_0}{\m^*} +n-2} + b_n \ \text{if}\ n\geq 2
	\end{cases}
	\label{eq:MAP_rate}\\
	& \text{where }b_n = \frac{(1 + \inv{n_0} - \frac{\m_0}{\m^*})^2}{2 (\frac{\m_0}{\m^*}+\frac{\max(0,n-2)}{n_0})(1 + \frac{n}{n_0} )} \; . \nonumber
\end{align}
\end{theorem}
Anticipating on \S\ref{ssec:bias-variance}, this inequality highlights an explicit $O(\frac{v}{n} + \frac{b}{n^2})$ variance-bias decomposition.
This inequality is derived with the symmetrized Bregman $\cB(a,b) + \cB(b,a)$ for which calculus is more tractable.
This explains why the variance term is twice larger than the asymptote~\eqref{eq:asymptote}.
Regarding the bias, it vanishes when $\frac{\m_0}{\m^*} =1 + \inv{n_0} $, which happens when the prior is slightly larger than the ground truth.
This correlates well with our numerical observations (cf App.~\ref{app:gaussian-variance}).

Note that if $X\sim \cN(0,\sigma^2)$, then $X^2 \sim \Gamma\paren{\half, \inv{2\sigma^2}}$ in the shape-rate parametrization of Gamma distributions. In fact the bounds above can be generalized to any distribution $\Gamma\paren{\alpha, \beta}$ with known shape $\alpha$.
This generalization encompasses exponential distribution when $\alpha=1$, as another important special case.
 We postpone these rates to \cref{app:gaussian-variance} for the sake of clarity.

\begin{figure*}[t]
	\centering
	\includegraphics[width=.8\textwidth]{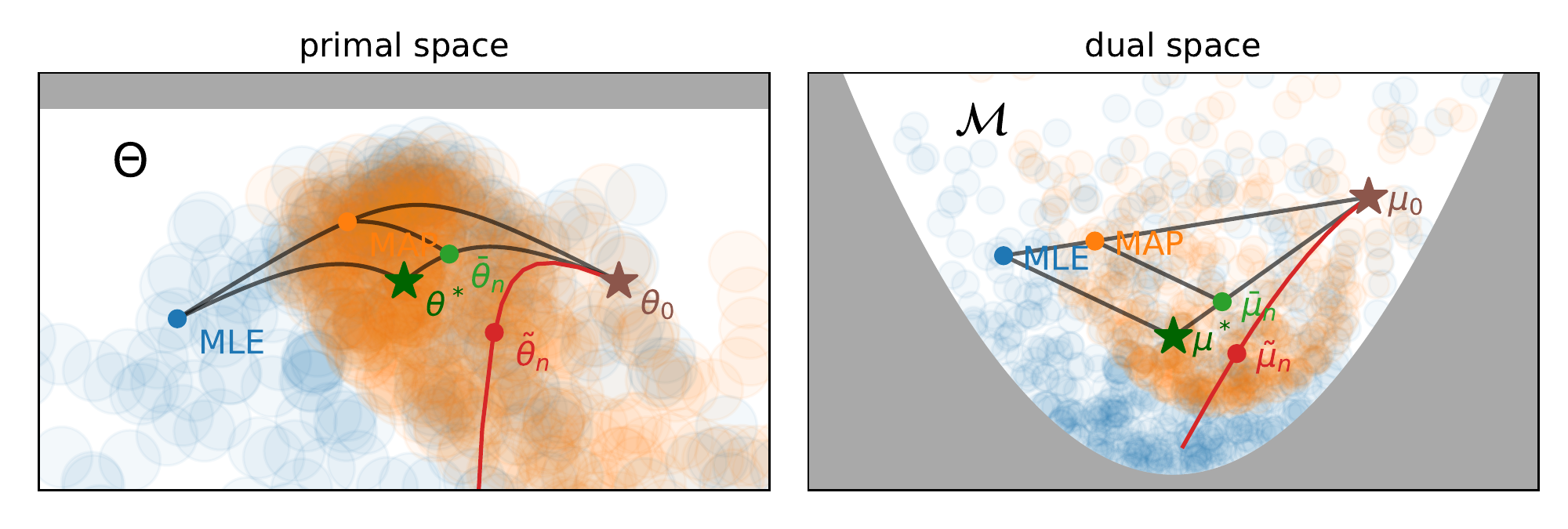}
	\caption{
	Primal and dual representations of a Gaussian $\cN(m,\sigma^2)$ MAP (blue) and MLE (orange) (\S\ref{ssec:gaussian} with $n=3$).
	In dual space, MAP is a scaled version of the MLE~\eqref{eq:defMAP} with expectation $\E[\hat\m_n^\text{MAP}]=:\bar \m_n$ (light green), and MLE is unbiased $\E[\hat\m_n^\text{MLE}]=\m^*$, as illustrated by the parallels in the grey triangle.
	In primal space, MAP has expectation $\tilde \nat_n$ (red), which intervenes in the bias-variance decomposition~\eqref{eq:bias-variance} from~\S\ref{ssec:bias-variance}.
	The hyperparameter of the prior $\nat_0$ controls the brown point's location while varying $n_0$ spans the long edges of the triangle and the red curve.
	Large blurry circles in the background are other instances of MAP and MLE revealing their distribution.
	}
	\label{fig:thales}
\end{figure*}

\subsection{Full Gaussian (Non-Trivial)}
\label{ssec:gaussian}
Now that we have solved the case of $\cN(0,\sigma^2)$, consider the full Gaussian $\cN(m,\sigma^2)$, which offers a highly non-trivial example for Problem 1.
Their log-likelihood reads $p(x) = -\frac{(x-m)^2}{2 \sigma^2} - \half \log(2\pi\sigma^2)$.
With sufficient statistic $T(x)=(x, x^2)$,
the mean parameters are $\m = \E[T(X)] = (m , m^2 + \sigma^2)$ belonging to the open set $\cM= \{(u,v) \cond u^2 < v\}$,
and the natural parameters are $\nat= (\frac{m}{\sigma^2} , \frac{-1}{2\sigma^2}) \in \Theta = \real \times \real_-$.
Examples of MAP and MLE  are represented in \cref{fig:thales} within $\cM$ and $\Theta$ delimited in grey.
Given these parameters, log-partition and entropy are, up to constants,
\alignn{
	\textstyle \logpart(\nat) &= \textstyle \frac{\nat_1^2}{-4\nat_2} - \half \log(-\nat_2) \\ 
	\textstyle \conj(\m) &= \textstyle - \half \log (\mu_2 - \mu_1^2)
}

These functions are neither smooth, nor strongly convex, but they are self-concordant, since $\conj$ is  the logarithmic barrier of a quadratic domain
\citep[p.177, example 4.1.1.4]{nesterov2003introductory}, and self-concordance is preserved by convex-conjugacy~\citep{nesterov1994interior} -- see more details in App.~\ref{app:gaussian}.
We now discuss the general problem and some ways to solve it via direct expansions of the Bregman divergence.

\section{PARTIAL SOLUTIONS}
\label{sec:insights}

\subsection{Asymptotic Rate}
\label{ssec:asymptote}
As a reference point for any finite convergence rate, it is interesting to briefly review the classical asymptotic behavior of these quantities as $n \rightarrow +\infty$.
Proofs are in App.~\ref{app:asymptote}, and \citet[\S1.1]{ostrovskii2021finite} offers a more comprehensive review.

Statistics typically give results on $\nat$, but the MAP~\eqref{eq:defMAP} is more simply expressed with $\meanp$, so let us focus on $\bregmanconj$.
Bregman divergences are locally quadratic, as seen via a second order Taylor expansion
\alignn{
    \textstyle \bregmanconj(\m^* ; \m)
    &\textstyle = \frac{1}{2}\norm{\m^* - \m}^2_{\mF}
    + O(\norm{\m - \m^*}^3),
    \label{eq:bregmanTaylor}
}
where the Mahalanobis norm  $\| x \|_{\mF}^2 = x^\top \mF x$  is induced by $\mF  := \nabla^2\conj(\m^*)$, the Hessian of the entropy at the optimum. It happens that  $\mF$ is also the inverse \textit{Fisher information matrix} at $\nat^*$, since
\aligns{
    \mF
    :=\nabla^2\conj(\m^*)
    = \nabla^2\logpart(\nat^*)^{-1}
    = \Cov_{\nat^*}[T(X)]^{-1}  \; .
}
Plugging the MLE~\eqref{eq:defMLE} or MAP~\eqref{eq:defMAP} into~\eqref{eq:bregmanTaylor}, we have
\begin{align}
	\label{eq:asymptote}
	\E \bregmanconj \left (\E [T(X)] ; \hat \meanp_n^\text{MLE/MAP} \right )
	= \frac{d}{2n} + O(n^{- \frac{3}{2}}) \; .
\end{align}
Both MLE and MAP have the same asymptote, as the contribution of the prior $n_0 \meanp_0$ gets negligible for large $n$.
This asymptote is independent of the optimum $\meanp^*$ or $\mF$ for well-specified models.
Next, we give another example for which we get a rate matching~\eqref{eq:asymptote}.

\subsection{Quadratic Case}
\label{ssec:quadratic}
As another classical reference point, we consider the case $\logpart(\nat) = \half \norm{\nat}_2^2$.
For instance, this is the log-partition of a Gaussian with known variance $I$,
\[
	\cX=\real^d,\quad \nu(dx) = \exp\paren{\textstyle \half[-\|x\|^2]} dx,\quad T(x)=x.
\]
In this case, $\conj(\meanp) = \half \norm{\meanp}_2^2$ as well, and both Bregman divergences are squared $\ell^2$ distances since
\begin{align}
	\bregmanconj(\meanp^* ; \meanp) = \half \norm{\meanp^* -  \meanp }_2^2  \; .
\end{align}
Thanks to the independence of samples, we can break down the MLE into individual point's contributions:
\begin{align}
	\expect{\half \norm{\m^* -  \inv{n}  \smallsum_i T_i}_2^2}
	=\frac{\Var(T)}{2n}
	=\frac{d}{2n}.
\end{align}
Adding a reference mean $\m_0$ to get the MAP yields
\begin{align}
		\!\!\!\!\! \expect{\half \norm{\m^* -   \MAPm^\text{MAP}}_2^2} \!
	&= \frac{n \Var(T) +  n_0^2 \norm{\m^* -  \m_0}^2}{2(n+n_0)^2}.\!
	\label{eq:MAP_quadratic}
\end{align}
We see here a variance term defining the $\frac{d}{2n}$ asymptote and a bias term in $O(n^{-2})$. However, this result does not generalize well to other families unless we make restrictive assumptions on $\conj$. 

{\bf If $\conj$ is $L$-Lipschitz} (e.g. $\logpart$ is defined within the $\ell^2$-ball of radius $L$), then
\begin{align}
    \bregmanconj(\m^* ; \m)
    &\leq L \norm{\m^* - \m} + \norm{\nat} \norm{\m^* - \m} \\
    &\leq 2L \norm{\m^* - \m},
\end{align}
so $\bregmanconj$ is Lipschitz, and~\eqref{eq:MAP_quadratic} yields a $O(\inv{\sqrt{n}})$ rate, but no common exponential families verify the assumption.

{\bf If $\conj$ is $L$-smooth}\footnote{$\conj$ is $L$-smooth iff $\nabla\conj$ is $L$-Lipschitz.} (e.g. $\logpart$ is $\frac{1}{L}$-strongly convex \citep{kakade2009duality}), then
\begin{align}
    \bregmanconj(\m^* ; \m)
    \leq \frac{L}{2} \norm{\m^* - \m}^2,
\end{align}
so $\bregmanconj$ is upper bounded by a quadratic, and we get~\eqref{eq:MAP_quadratic} as an upper bound.
It is also possible to get (more complex) upper bounds under restricted notions of strong-convexity \citep{negahban2012unified}.
Besides the Gaussian with known variance, the problem is that no standard exponential family has a \textit{globally} strongly convex log-partition function. The next section focuses on \textit{local} quadratic behavior, which is more realistic.

\subsection{Locally Quadratic Case}
\label{ssec:local-quadratic}
From the Taylor expansion~\eqref{eq:bregmanTaylor},
we know that all Bregman divergences are locally quadratic.
Under some assumptions, such as self-concordance\footnote{
In 1d, $f$ is self-concordant iff $\forall x, \abs{f'''(x)} \leq 2 \abs{f''(x)}^{\frac{3}{2}}$.
} of $\conj$ \citep[Ch.~4.1]{nesterov2003introductory}, we can quantify when this quadratic behavior kicks in. Proofs for this subsection are in App.~\ref{app:self-concordant}.
\begin{proposition}
	\label{prop:selfConcordant}
Let $\conj:\cM\rightarrow \real$ be a self-concordant convex function, $\m, \m^* \in\cM$ and $\mF = \nabla \conj(\m^*)$. Then\footnote{$0.21$ is a value of $x$ such that $x^2 \geq -\frac{x}{1-x} - \log(1 - \frac{x}{1-x})$.}
\aligns{
	\norm{\m^*-\m}_{\mF} < 0.21
	\implies
	\bregmanconj(\m^*,\m) \leq \norm{\m^*-\m}_{\mF}^2.
}
\end{proposition}
To gain insights into how many samples are needed, we can estimate when $\expecti{\norm{\m^*-\m}_{\mF}} < 0.21 $.
For the MLE, the proof of~\eqref{eq:asymptote} from~\eqref{eq:bregmanTaylor} yields $\expecti{\norm{\m^*-\hat \m_n}^2_{\mF}} = \frac{d}{n}$ in general, so a sufficient condition is $n \geq 25 d$.
For MAP, transforming~\eqref{eq:MAP_quadratic}, we get the sufficient condition $n\geq 25d + 5 \norm{\m^* -  \m_0} - n_0$.
This means that on average, we need $25$ times more samples than the dimension to reach the quadratic regime and ensure an upper-bound like~\eqref{eq:MAP_quadratic}.

The entropy $\conj$ is self-concordant for several common families such as Gaussians (\S\ref{ssec:gaussian}), and all families with $\logpart \approx -\log$ such as
exponential distributions,
Laplace with known mean,
Pareto with known minimum value,
or Weibull with known shape $k$.
The entropy is also self-concordant when $T$ lives in a compact \citep{bubeck2015entropic} -- e.g., categorical and Dirichlet distributions.
Precisely, categorical variables illustrate that Proposition~\ref{prop:selfConcordant} does not imply directly a bound on the \emph{expected} Bregman.
The expected KL of the categorical MLE is always infinite, as previously mentioned in \S\ref{sec:problem}.

Overcoming this limitation,
\citet{ostrovskii2021finite} characterize the number of samples needed to be upper bounded by a quadratic \emph{with high-probability}, for any parametric models with a self-concordant log-likelihood $f$.
\citet{anastasiou2017bounds} obtains a similar flavor of result under other assumptions on the third derivative of $f$.
More closely, in the world of exponential families, \citet{kakade2010learning} prove a result similar to Proposition~\ref{prop:selfConcordant} from a local bound on all higher-order moments of $\logpart$ in $\nat^*$.
However, these results are expressed with quadratics in $\nat$, not $\m$, and they do not directly translate to convergence rates for the MAP, but they might with some more work.

More generally, the present proposition and these related works answer~\eqref{problem} only \emph{partially}, as they all give \textit{large sample} results, that hold when $n\geq N$ for some constant $N$. 
A full solution to~\eqref{problem} would apply to small $n$.
Informed by the properties that we have seen so far, we next investigate a general decomposition of the Bregman that could guide us towards a solution.

\subsection{Bias-Variance Decomposition}
\label{ssec:bias-variance}
In both the quadratic~\eqref{eq:MAP_quadratic} and the Gaussian variance examples~\eqref{eq:MAP_rate}, the upper bound takes the form $O(\inv{n}) + O(\frac{\text{bias}}{n^2})$,
giving us a flavor of what we would like as a general result for exponential families: a finite sample convergence rate, with variance and bias terms that reflect the important constants of the problem.
Such a decomposition exists for any Bregman divergence \citep[Theorem 0.1]{pfau2013generalized}.
\begin{theorem}[Bregman Bias-Variance Decomposition]
	Let $\tilde \theta_n := \expecti{\hat \theta_n}$ be the expectation of the MAP in primal space, and $\tilde \m_n = \nabla \logpart(\tilde \theta_n )$ be its dual representation. The  expected Bregman decomposes into
\begin{equation}
	\expect{\bregmanconj(\m^* ; \hat \m_n)} = {\bregmanconj(\m^* ; \tilde \m_n)}
	+ {\expect{\bregmanconj(\tilde \m_n ; \MAPm)}}
	\label{eq:bias-variance}
\end{equation}
\end{theorem}
We plot this decomposition for $\cN(\mu, \sigma^2)$ in Fig.~\ref{fig:gaussian_decomposition} , and we illustrate the primal mean $\tilde \nat_n$ in  Fig.~\ref{fig:thales}.

\textbf{Remark:} In this decomposition, the primal expectation $\expecti{\hat \theta_n}$ is the reference point.
An estimator will be unbiased if $\tilde \nat_n = \nat^*$.
This is not true for the MLE, which is unbiased w.r.t. the dual parameter $\expecti{\hat \m_n}=\m^*$.

We show in App.~\ref{app:bias-variance} that the bias decreases like ${\bregmanconj(\m^* ; \tilde \m_n)} \leq \frac{2}{n(n-2)}$ for Gaussian variance MLE, and
 ${\bregmanconj(\m^* ; \tilde \m_n)} \leq \frac{ \norm{\m^* - \m_0}^2 }{(1 + \frac{n}{n_0})^2 }$ for a quadratic MAP.
These observations hint towards a general $O(1/n^2)$ upper bound for the bias, while the variance may be less dependent on the initialization $\nat_0$.

\begin{figure}[t]
	\centering
	\includegraphics[width=.4\textwidth]{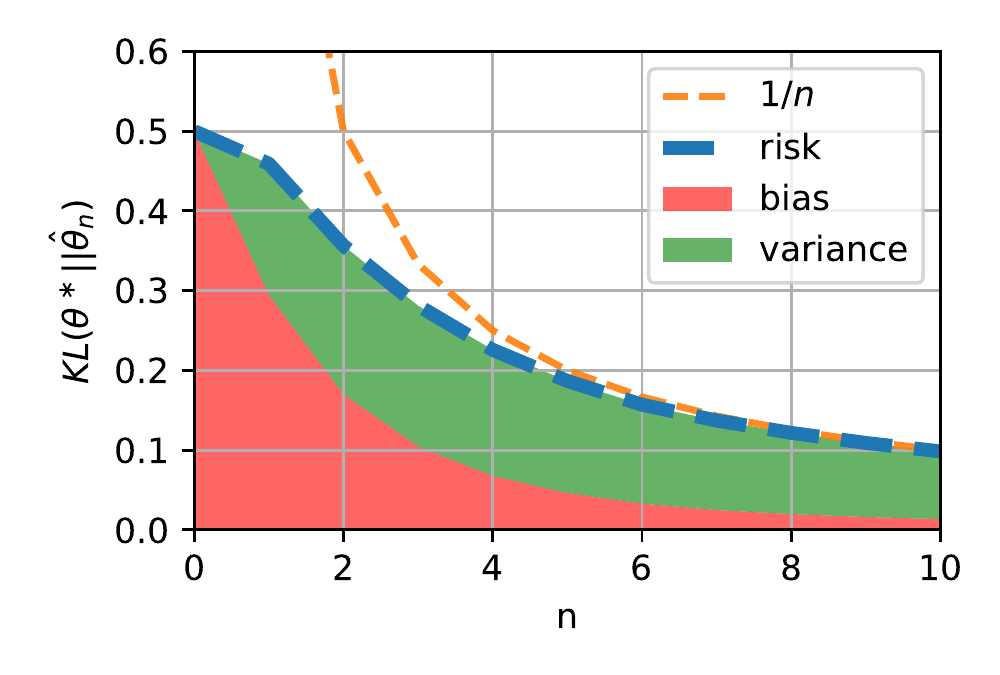}
	\caption{
	Bias-Variance Decomposition for a Gaussian $\cN(m, \sigma^2)$ with $\meanp^*=(0, 1), \meanp_0 = (1,2)$ and $n_0=1$. The asymptote is $\frac{1}{n}$.
	}
	\label{fig:gaussian_decomposition}
\end{figure}

In this section, we considered direct expansions of~\eqref{eq:bregmanMAP}.
None of them could fully solve~\eqref{problem}.
Next, we investigate whether an optimization approach could solve it.

\section{AN OPTIMIZATION PROBLEM}
\label{sec:optimization}

As we saw in \S\ref{sec:problem}, MAP  can be interpreted as stochastic mirror descent (SMD).
This means that \textbf{1)} we may obtain a convergence rate for MAP from an optimization analysis, and \textbf{2)} any insights gained from MAP may inform other designs and analyses of SMD.
In particular, we know that MAP converges asymptotically as $O(n^{-1})$, so we hope to find a convergence rate for SMD that could capture this behavior. 
We first review the assumptions of relative smoothness, helpful to deal with non-smooth functions, before investigating recent analyses of SMD with the MAP.

\subsection{Relative Smoothness}
Mirror descent (MD) \citep{nemirovski1983problem,beck2003mirror}, also known as
Bregman (proximal) gradient, relative gradient descent or NoLips,
and SMD \citep{nemirovski2009robust,ghadimi2012optimal}
are typically encountered in non-smooth (online) optimization,
under bounded (or Lipschitz) gradient assumption on the objective $f$
and strong convexity assumption on the potential $\logpart$
\citep[Th. 4.2(MD) \& Th. 6.3(SMD)]{bubeck2015convex}.
In our case, these assumptions do not hold.
For instance $\logpart = -\log$ is neither smooth nor strongly convex.

Recently, these assumptions have been relaxed to the $\stgcvx$-strong convexity and $\smooth$-smoothness of $f$
\emph{relative} to a reference function $\logpart$, defined as
\aligns{
	\stgcvx \cB_{A}(x, y)
	\leq
	\cB_f(x,y)
	\leq
	\smooth \cB_A(x,y) \; .
}
When $\logpart = \norm{\cdot}^2$, we recover the standard smoothness and gradient descent.
These conditions ensure the linear convergence of MD with mirror map $\nabla A$
\citep{birnbaum2011distributed, bauschke2017descent, lu2018relatively},
even when $f$ is not smooth, and $\logpart$ not strongly convex.

For exponential families, MAP perfectly fits into this framework, as
\aligns{
	f(\theta) = A(\theta) - \expect{\lin{T(X), \theta}}
}
is $1$-smooth and $1$-strongly convex relative to $A$.
Our goal is then to find an applicable convergence rate for SMD under relative smoothness.

\subsection{Bounding the Randomness}

To analyze stochastic algorithms, one also needs to quantify the randomness of stochastic gradients $g(\nat)$.
While many assumptions exist for SGD \citep[\S3 for a modern review]{khaled2020better}, only a few have been adapted to SMD with relative smoothness \citep{hanzely2018fastest, dragomir2021fast, dorazio2021stochastic}, but they have so far been lacking concrete examples.
We review these analyses in the light of the MAP and provide a summary in \cref{tbl:assumptions}.

\begin{table}[t]
	\newcommand*{\greencmark}{\textcolor{Green}{\cmark}}
	\newcommand*{\redxmark}{\textcolor{Red}{\xmark}}
	\caption{Summary of results for SMD
		under relative smoothness and relative strong convexity assumptions.
		Each row correspond to one analysis, and each columns answers one question.
		($-\log$) does the bound hold for the Gaussian variance example (\S\ref{ssec:gaussian-variance})?
		($\lr_n \sim \inv{n}$) does it converge with a $O(\inv{n})$ step-size?
		($f$) is the bound in function value, or in reverse Bregman $\bregman(\nat^*,\hat\nat_n)$?
		($\hat\nat_n$) is it for the last iterate or an average ?
		None of these analysis check all the boxes needed to address~\eqref{problem}.
	}
	\begin{center}
		\begin{tabular}{lcccc}
			\toprule
			Boundedness & $-\log$ &  $\lr_n \sim \inv{n}$ & $f$ & $\hat\nat_n$ \\
			\midrule
			Variance on $\Theta$~\eqref{eq:hanzely} 
			& \redxmark & \greencmark & \greencmark  & \redxmark
			\\
			Variance at $\theta^*$~\eqref{eq:dragomir} 
			& \redxmark & \greencmark & \redxmark  & \greencmark
			\\
			Optimization gap~\eqref{eq:dorazio} 
			& \greencmark & \redxmark & \redxmark & \greencmark
			\\
			\bottomrule
		\end{tabular}
	\end{center}
	\label{tbl:assumptions}
\end{table}


\subsubsection{Analogs of the Variance}
Let us introduce the symmetrized Bregman induced by $\conj$, written $\cS_\conj(\m_1, \m_2) = \bregmanconj(\m_1, \m_2) + \bregmanconj(\m_2, \m_1)$.
\citet{hanzely2018fastest} assume that the expectation of $\cS$ between stochastic and deterministic updates verifies
\alignn{
	\expect[g]{\cS_\conj \paren{
			\hat\m_{n} - \lr g(\hat \nat_{n}),
			\hat\m_{n}   -\lr \nabla f(\hat \nat_{n})
	}} \leq \lr^2 C
	\label{eq:hanzely}
}
for all possible iterates $\MAPt$, relevant step-sizes $\lr$ and for some constant $C$.
When $A(\theta) = \frac{1}{2}\norm{\theta}^2$,
this definition recovers the variance of the stochastic gradient
\[
\expect[g\!\!]{\norm{\nabla f(\theta) - g(\theta)}^2}\leq C.
\]
Under this assumption, \citet[Lem.4.8]{hanzely2018fastest} prove a $O(1/n)$ convergence rate on function values with $O(1/n)$ step-sizes and tail averaging \citep{lacostejulien2012simpler} in primal space $\Theta$.

\citet{dragomir2021fast} define the weaker assumption
\alignn{
	\expect[\tilde g]{
		\cB_{A^*}(\MAPm - 2\lr g(\theta_*), \MAPm)
	} \leq 2 \lr^2 C \; .
	\label{eq:dragomir}
}
When $A(\theta) = \frac{1}{2}\norm{\theta}^2$,
we recover the variance of the gradients at the optimum
\[
\expect[g\!\!]{\norm{\nabla g(\theta^*)}^2}\leq C.
\]
Using their descent lemma \citep[Eq. (12)]{dragomir2021fast} with the $O(1/n)$ step-size used by \citet[Th. 3.2]{gower2019sgd} for SGD, we obtain a $O(1/n)$ convergence rate, on the Bregman with \emph{reversed} arguments $\bregman(\nat^*, \hat\nat_n)$.

These two analyses seem promising for~\eqref{problem}, but none of these assumptions hold in front of barrier objectives such as the $-\log$ from \S\ref{ssec:gaussian-variance}.
Indeed, they both assume their bound holds uniformly for every possible iterate $\hat \nat_n$.
Yet $\cN(0,\sigma^2)$ has a positive mass around $0$.
This means that $\hat \m_n$ can get arbitrarily close from $0$, where the $-\log$ is unbounded, along with the associated  Bregman divergences~\eqref{eq:hanzely} and~\eqref{eq:dragomir}.
In general, this uniform bound over $\hat \m_n$ cannot hold for \emph{barrier} objectives -- functions exploding to infinity in some finite point of space.

Both of their proofs hold if we add an expectation over $\MAPm$ to their assumption.
However, this is not helpful, as verifying the assumption becomes as hard as the initial problem.
For instance, the expectation of~\eqref{eq:hanzely} over $\MAPm$ is an upper bound on the variance term of~\eqref{eq:bias-variance} (cf App.~\ref{app:bias-variance}).
Confronted with this difficulty, we investigate an alternative definition of variance.

\subsubsection{Bounded Optimality Gap}
Inspired by \citet{loizou2021stochastic}, \citet{dorazio2021stochastic} explore the hypothesis
\alignn{
	\min_\nat f(\nat) - \expect[X]{\min_\nat f_X(\nat)} \leq C,
	\label{eq:dorazio}
}
where $f_X$ is a stochastic estimate of $f = \expecti{f_X}$. In our case $f_X(\nat) = - \log p(X\cond \nat)$.
In other words, this lower bounds the expectation of the minimum of the stochastic estimates.
For probabilistic models, such a bound is finite as soon as the model cannot give infinite density to any data point $x$.
This holds, for instance, for discrete distributions because the probability mass is upper bounded by $1$; however, it rules out many families.
In the case of normal distributions $\cN(m, \sigma^2)$, setting $m=x$ and $\sigma^2 \rightarrow 0$ gets $p_\nat (x) \rightarrow +\infty$,.
We have a similar behavior for gamma distribution with $\alpha = \beta x$ and $\beta \rightarrow +\infty$, or with the beta distribution with $\alpha=\beta \frac{x}{1-x}$ and $\beta \rightarrow +\infty$.
Other counter-examples include inverse Gaussians, log-normal, gamma, inverse gamma.

It is possible to overcome this limitation by treating batches of samples as single samples by averaging sufficient statistics, e.g.,. $Y = \{X_1, \dots, X_k\}$ and $T(Y) = \inv{k}\sum_i T(X_i)$.
For instance, a multivariate normal of dimension $d$ cannot attribute infinite density to $d+1$ samples that are not in an affine subspace.

Overall, \eqref{eq:dorazio} can partially handle barrier objectives, but it fails to account for the step-size $\lr_n = \inv{n_0+n}$, as \citet[Thm.1]{dorazio2021stochastic} only proves linear convergence to a variance ball of size $\frac{C}{\stgcvx}$ under constant step-size.
This is in contrast with~\citet{dragomir2021fast} which can handle decreasing step-sizes but not barrier objectives.
Proving convergence of stochastic mirror descent on barrier loss remains an open problem.

\section{CONCLUSION}
Despite the MLE and MAP estimators in the exponential family being classical and known in statistics for decades, we highlighted in this paper open problems to bound their frequentist risk (the expected KL) in a non-asymptotic way. We reviewed some partial results, such as a large sample analysis that describes how many samples are needed to ensure a locally quadratic regime~\citep{kakade2010learning, ostrovskii2021finite} for which rates are known. We also related this problem to the one of obtaining convergence rates in stochastic optimization, observing that MAP fits the framework of stochastic mirror descent with relative smoothness assumptions.
Nevertheless, none of the current analyses of SMD hold for the MAP, even on a simple family such as $\cN(0,\sigma^2)$, thus revealing an area for progress in non-Euclidean optimization.
In writing this paper, we hope to attract attention to this fundamental problem, leading to progress in both optimization and statistics.

\subsubsection*{Acknowledgments}
We express our gratitude to Reza Babanezhad, Radu Dragomir and Hadrien Hendrikx for their insights on stochastic mirror descent. 
This work was partially supported by the NSERC Discovery Grant RGPIN2017-06936 and the Canada CIFAR AI Chair Program. 
Simon Lacoste-Julien is a CIFAR Associate Fellow in the Learning in Machines \& Brains program.

\bibliographystyle{apalike}
\bibliography{references}

\clearpage
\appendix
\onecolumn

\addcontentsline{toc}{section}{Appendix} 
\part{Appendix} 
\parttoc 

\section{PROOFS FOR GAUSSIAN VARIANCE}
\label{app:gaussian-variance}

In this section, we prove the results mentioned in \S\ref{ssec:gaussian-variance},
and add some context and experimental observations.
As mentioned in the main text, the centered gaussian $\cN(0,\sigma^2)$ has sufficient statistic $T(X)=X^2$ which follows a gamma distribution $\Gamma\paren{\half, \inv{2 \sigma^2}}$. 

In general, if $X$ is part of the exponential family, then $T(X)$ is part of the natural exponential family with the appropriate support and base measure, with the same log-partition function as $X$ up to constants.
MLE and MAP only depend on $T(X)$, not $X$, so their performance only depends on the distribution of $T(X)$.
 
In this section we derive results for samples  from a general gamma distribution $X \sim \Gamma(\alpha,\beta)$ with known shape parameter $\alpha$, but unknown rate parameter $\beta$.
Results for the Gaussian follow by taking $\alpha=\half$,
We also immediately get results for exponential distributions by taking $\alpha=1$.
For instance for the MLE we derive the following theorem:

\begin{theorem}[MLE Upper Bound]
	\label{thm:gammaMLE}
	Consider an exponential family such that $T(X)$ is a gamma $\Gamma(\alpha, \beta)$ with known shape $\alpha$.
	the expected KL between $\mu_*$ and the MLE $\hat\mu_n$	is infinite when $\alpha n\leq 1$ and otherwise upper bounded by
	\begin{align}
		 \expect{\bregmanconj( \mu_*; \hat \mu_n) }
			\leq \inv{2n} + \frac{1}{n(n \alpha-1)}  \; .
	\end{align}
\end{theorem}
	
To obtain the result for Gaussian variance (see Theorem~\ref{thm:varianceMLE}), it suffices to set $\alpha=\half$ in Theorem~\ref{thm:gammaMLE}.

In this section, we review useful properties of the gamma distribution and associated Bregman divergence in \S\ref{app:gamma}. 
Then we prove theorem~\ref{thm:gammaMLE} in \S\ref{app:gammaMLE}.
Then we prove an extension of theorem~\ref{thm:varianceMLE} in \S\ref{app:multivariateMLE}, and prove a useful lemma about the expectation of the natural parameter of the MAP in \S\ref{app:nat-bound}, in order to prove upper bounds for the MAP in \S\ref{app:MAP-bound}.
Finally we numerically investigate the effect of prior hyper-parameters in \S\ref{app:prior-choice}.

\subsection{Gamma Distribution}
\label{app:gamma}
The density of $\Gamma(\alpha,\beta)$ reads
\alignn{
	p(x) = \frac{\beta^\alpha}{\Gamma(\alpha)} x^{\alpha-1} e^{-\beta x} \; .
}
When $\alpha$ is known, it can be cast as an exponential family~\eqref{eq:def_expfamily} with sufficient statistic $T(x)=x$, domain $\cX = \real_+$ and base measure $\nu(x)\propto x^{\alpha-1}$. Then the natural parameter is $\nat = -\beta < 0$ and the log-partition function is
\alignn{
	\logpart (\nat) = -\alpha \log(-\nat)  + \log\Gamma(\alpha) \; .
}
From there we find that the mean parameter is $\mu = \frac{\alpha}{-\nat} >0$ and the entropy has the same form as the log-partition $\conj(\mu) = - \alpha \log(\mu) + \cst$.
This means that the primal and dual Bregman divergences have the same form as well
\begin{align}
	\bregmanconj(\m_*; \m_n) 
	&= \alpha \left ( \frac{\m_*}{ \m_n} - 1 - \log  \frac{\m_*}{ \m_n} \right) 
	= \alpha \phi\paren{\frac{\m_*}{ \m_n}},
	\\
	\bregman( \nat_n; \nat_* ) 
	&=  \alpha \left ( \frac{ \nat_n}{\nat_*} - 1 - \log  \frac{ \nat_n}{\nat_*} \right)
	= \alpha \phi\paren{\frac{ \nat_n}{\nat_*}},
\end{align}
\begin{wrapfigure}[14]{r}{0.4\textwidth}
	\centering
	\includegraphics[width=.38\textwidth]{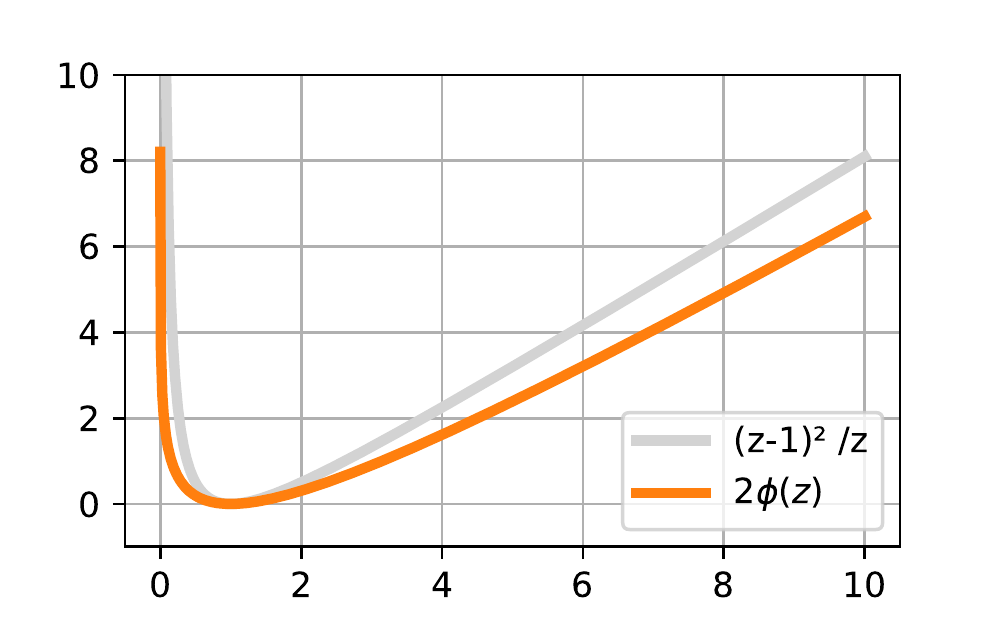}
	\caption{%
	Illustration of $\phi(z)$ (orange) and its upper bound $\phi(z) + \phi\paren{z^{-1}}$ (grey).
	They both are barriers near $0$.
	}
	\label{fig:phi}
\end{wrapfigure}
where these 2 lines are equal, and $\phi$ measures the discrepancy between the ratio $\frac{ \nat_n}{\nat_*} =  \frac{\m_*}{ \m_n}  $ and $1$ via the function
\begin{align}
	\phi(z) := z - 1 - \log(z),
\end{align}
illustrated in orange in Figure~\ref{fig:phi}.
To derive the upper bound for the MAP, due to the difficulty of finding a closed form for the expectation of the logarithm, we focus on the symmetrized Bregman instead
\alignn{
	\bregmanconj( \mu_*; \mu_n) 
	\leq \cS_{\conj} (\mu_*, \mu_n )
	&:= \bregmanconj( \mu_*; \mu_n)  + \bregmanconj( \mu_n; \mu_*) \nonumber\\
	&=  \alpha \paren{\frac{ \mu_*}{\mu_n} -1 +\frac{ \mu_n}{\mu_*} - 1 },
	\label{eq:symmetrized_bregman} \\
	&= \alpha\phi\paren{\frac{\mu_*}{\mu_n}} + \alpha\phi\paren{\frac{\mu_n}{\mu_*}}
}
Writing $z=\frac{\mu_*}{\mu_n}$ this is equivalent to
\aligns{
	\label{eq:log_bound}
	\phi(z) \leq \phi(z) + \phi\paren{z^{-1}}
	&=  z -1  + \inv{z} -1
	= \frac{(z-1)^2}{z} \; ,
}
and as illustrated by the grey upper bound in Figure~\ref{fig:phi}.

\subsection{Proof for the MLE}
	\label{app:gammaMLE}
	
	\begin{proof}
	Since $T(X)$ follows a gamma distribution $\Gamma(\alpha, \beta)$, the MLE is a scaled sum of gammas $\hat \mu_n = \inv{n} \sum_i T(X_i)$.
	As such it is also a gamma with parameter $\Gamma(n \alpha, n \beta)$ and expectation $\frac{n\alpha}{n\beta} = \frac{\alpha}{\beta} = \mu_*$.
	If we consider the ratio $\frac{\MAPm}{\mu_*}$, it is also a gamma with parameter $\Gamma(n \alpha, n \alpha)$.
	Its inverse follows an inverse gamma distribution with expectation
	\begin{equation}
		\expect{\frac{\mu_*}{\hat \mu_n}} 
		= \begin{cases}
			\frac{n\alpha}{n \alpha -1 }\ \text{ if } n\alpha > 1, \\
			+\infty \  \text{ otherwise. }
		\end{cases}
	\end{equation}
	which implies that for $n\alpha>1$, 
	\begin{align}
		\expect{\frac{\mu_*}{\hat \mu_n}} -1  
		= \frac{n \alpha}{n\alpha -1 } - 1
		= \frac{1}{n \alpha -1}
	\end{align}
	There is also a closed form solution for the expected logarithm of a gamma.
	Indeed, the sufficient statistic of a gamma is $(X, \log(X))$, so one can apply formula~\eqref{eq:mirror-map} on the log-partition of a gamma to get
	\begin{align}
		\expect{\log \frac{\hat \mu_n}{\mu_*}} 
		= \psi\paren{n \alpha} - \log\paren{n \alpha},
	\end{align}
	where $\psi$ is the \href{https://en.wikipedia.org/wiki/Digamma_function}{digamma function.}
	Consequently the suboptimality of the MLE has a closed form solution
	\begin{align}
	\expect{\bregmanconj( \mu_*; \hat \mu_n) }
	&= \alpha \expect{\frac{\mu_*}{ \hat \mu_n} - 1 + \log \left(\frac{\hat \mu_n}{\mu_*} \right) },
	\\
	& =
	\begin{cases}
		\alpha \paren{ \frac{1}{n\alpha - 1} + \psi\paren{n \alpha} - \log\paren{n \alpha} }, \ \text{ if } n\alpha>1, \\
			+\infty \  \text{ otherwise. }
	\end{cases}
	\end{align}
	Surprisingly, for a gaussian variance where $\alpha = \half$, we need $3$ samples or more for the expected loss to be bounded.
	When the expectation is finite, we can get a more interpretable formula using known \href{https://en.wikipedia.org/wiki/Digamma_function#Inequalities}{bounds on the digamma function},
	\begin{align}
		-\inv{x} \leq \psi(x) - \log(x) \leq -\inv{2x} = - \inv{x} + \inv{2x}	\; ,
	\end{align}
	giving, for $n \alpha >1 $,
	\begin{align}
		\frac{\alpha}{n\alpha-1} - \frac{\alpha}{n\alpha}
		&\leq \expect{\bregmanconj( \mu_*; \hat \mu_n) }
		\leq \frac{\alpha}{n\alpha-1} - \frac{\alpha}{n\alpha} + \frac{\alpha}{2n\alpha}
		\\
		\iff
			\frac{1}{n(n \alpha-1)}
			&\leq \expect{\bregmanconj( \mu_*; \hat \mu_n) }
			\leq \inv{2n} + \frac{1}{n(n \alpha-1)} \; ,
	\end{align}
	so we get a $\Omega(n^{-2})$ lower bound and a $O(n^{-1}) + O(n^{-2})$ upper bound.
	\end{proof}

\subsection{Multivariate MLE}
\label{app:multivariateMLE}

For the sake of simplicity, in higher dimension we focus on Gaussian covariance estimation and avoid the general Wishart discussion. 
In higher dimensions, $X \sim \cN(0, \mu_*)$, $X\in\real^d$, $T(X) = XX^\top$, and the mean parameter $\mu_*$ 
is a $d \times d$ symmetric, positive definite covariance matrix with $p = \frac{d(d+1)}{2}$ degrees of freedom.
Note that here $d$ denotes the dimensionality of the data $X$, rather than the dimensionality of the parameters $\mu_*$.

\begin{theorem}[Multivariate MLE Upper Bound]
	The MLE of  the covariance matrix of $X_i\sim\cN(0,\mu_*)$ is $\hat \mu_n = \inv{n} \sum_i X_i X_i^\top \in\real^{d\times d}$ with $p = \frac{d(d+1)}{2} $ degrees of freedom. 
	The expected KL divergence between $\mu_*$ and $\hat\mu_n$ 
	is infinite when $n\leq d+1$ and otherwise upper bounded by
	\alignn{
	 	\expect{\bregmanconj( \mu_*; \hat \mu_n^\text{MLE(d)})}
 	\leq \frac{p}{2n} + \frac{p(d+2)}{n (n - d -1)} 
 	\ , \forall n >d+1 \;.
}
\end{theorem}
We see from~\eqref{eq:asymptote} that this bound is asymptotically tight.
\begin{proof}
The entropy of $X$ is a negative log-determinant $\conj(\mu) = -\log\det(\mu)$, whose gradient is the negative matrix inverse $\nabla\conj (\mu) = - \mu^{-1}$.
The associated Bregman divergence is
\alignn{
	\bregmanconj(\mu_*; \hat \mu_n) 
	= \half \paren{\Tr(\mu_* \hat \mu_n^{-1} ) - d  - \log\det(\mu_* \hat \mu_n^{-1}) } \; .
}
Thanks to the linearity of the trace, the expectation becomes
\alignn{
	\expect{\bregmanconj(\mu_*; \hat \mu_n) }
	= \half \paren{\Tr(\expect{\mu_* \hat \mu_n^{-1}} ) - d   
	+\expect{\log\det( \hat \mu_n \mu_*^{ -1}))} } \; .
}
When the estimator is the MLE, $\hat \mu_n = \inv{n} \sum_i X_i X_i^\top$, 
then we define the mean parameter ``ratio'' as
\aligns{
	\mV &:= n \mu_*^{-\half}\hat \mu_n \mu_*^{-\half} = \sum_i (\mu_*^{-\half} X_i) (\mu_*^{-\half}X_i)^\top,
}
such that $\Tr(\mu_* \hat \mu_n^{-1})  = n \Tr(\mV^{-1})$.
But $\mu_*^{-\half}X_i \sim \cN(0,\mI)$, so that $\mV$
is sampled from a Wishart  $\cW(\mI, n)$, where $\mI$ stands for the identity matrix of order $d$.
Recall that $\E[\mV] = n \mI$.
Thanks to $\log\det$ being a sufficient statistic of the Wishart, and the natural to mean parameter formula~\eqref{eq:mirror-map}, we have a closed form for the expected log-determinant of a Wishart 
\aligns{
	\E[\log\det \mV] &= \psi_d\paren{\half[n]} + d \log 2 \; ,
}
where $\psi_d(\half[n]) = \sum_{i=0}^{d-1} \psi ( \half[n - i])$ is the multivariate digamma function.
The expectation of an inverse Wishart  is straightforward to compute from the density and the log-partition function
\begin{equation}
	\E[\mV^{-1}] = \begin{cases}
		+\infty \text{ if } n\leq d+1 \\
		 \frac{\mI}{n - d - 1} \text{ otherwise,} 
	\end{cases}
\end{equation}
which proves the infinite part of the statement.
Consider now the case $n>d+1$. Using
\alignn{
	\Tr(\expect{\mu_* \hat \mu_n^{-1}}) - d
	= n\Tr(\E[\mV^{-1}]) - d 
	= \frac{nd}{n - d - 1} -d
	= \frac{d (d+1)}{n -d -1}
	= \frac{2p}{n -d -1},
}
and putting it all together, we get the following closed form for the expectation of the divergence
 \begin{align}
 	\label{eq:where-we-put-it-back-together-multivariate-gaussian}
 	\expect{\bregmanconj( \mu_*; \hat \mu_n^\text{MLE(d)})}
 	= \frac{p}{n-d-1} + \half \left ( \psi_d\paren{\half[n]} - d \log\paren{\half[n]}  \right ).
 \end{align}
To bound $\psi_d$, we can use the same bound as in the univariate case, $\psi(x) \leq \log(x) - \inv{2x}$, to get
\begin{align}
	\psi_d\paren{\half[n]} - d \log\paren{\half[n]}
	&= \sum_{i=0}^{d-1} \paren{\psi \paren{\half[n - i]} - \log\paren{\half[n]} } \\
	&\leq \sum_{i=0}^{d-1} \paren{\log\paren{\half[n-i]} - \inv{n-i} - \log\paren{\half[n]} } \\
	&= \sum_{i=0}^{d-1} \paren{\log\paren{1 - \frac{i}{n}} - \frac{1}{n-i}} \; .
	\label{eq:psid_bound}
\end{align}
We can bound sum of reciprocals $\sum_{i=0}^{d-1} \frac{1}{n-i}$ by the typical bound on 
the harmonic sum $H_n = \sum_{k=1}^n \frac{1}{k}$,
\aligns{
	\inv{2n +1}
	\leq H_n - \log(n) - \gamma \leq 
	\inv{2n - 1},
}
where $\gamma$ is the Euler constant. Then
\begin{align}
	- \sum_{i=0}^{d-1} \frac{1}{n-i} 
	= H_{n-d} - H_n
	&\leq \log(n-d) + \gamma + \inv{2(n-d) - 1} - \log(n) - \gamma - \inv{2n+1} \\
	&= \log\paren{1 - \frac{d}{n}}  + \frac{2(d + 1)}{(2n+1)(2(n - d) - 1)} \\
	&< \log\paren{1 - \frac{d}{n}}  + \frac{d + 1}{2n(n - d - 1)} \; .
\end{align}
Plugging this back into~\eqref{eq:psid_bound} yields
\begin{align}
	\psi_d(\half[n]) - d \log(\half[n]) 
	\leq \sum_{i=0}^{\textcolor{red}{d}} \log\paren{1 - \frac{i}{n}}
	+  \frac{d + 1}{2n(n - d - 1)} \; ,
\end{align}
where the sum now goes up to $d$.
To bound the remaining sum, we use $\log(1+x)\leq x$ (for $x > -1$) to get
\begin{align}
	\sum_{i=0}^{d} \log\paren{1 - \frac{i}{n}}
	\leq \sum_{i=0}^{d} - \frac{i}{n}
	 = -\frac{d(d+1)}{2n} = -\frac{p}{n} \; .
\end{align}
Putting those bounds together in 
\cref{eq:where-we-put-it-back-together-multivariate-gaussian}
and reorganizing yields
\begin{align}
 	\expect{\bregmanconj( \mu_*; \hat \mu_n^\text{MLE(d)})}
 	&< p \paren{ \inv{n-d-1} - \inv{2n}}  + \frac{d + 1}{4n(n - d - 1)}\\
 	&= p \frac{n + d +1}{2 n (n-d-1)}  + \frac{d(d + 1)/2}{2dn(n - d - 1)}\\
 	&= p\frac{n-d-1}{2n(n-d-1)} + p\frac{2(d+1)}{2 n (n-d-1)} + \frac{p}{2dn (n - d - 1)}\\
 	&= \frac{p}{2n} + \frac{p(d+ 1 + \inv{2d})}{n (n-d-1)} \\
 	&\leq \frac{p}{2n} + \frac{p(d+2)}{n (n - d -1)}
 	\ , \forall n >d+1 \;.
\end{align}
where $p = \frac{d(d+1)}{2}$ and the last inequality used $\inv{2d}\leq 1$. 
\end{proof}

\subsection{Bounding the expected natural parameters for the MAP}
\label{app:nat-bound}
Before proving a convergence rate for the MAP, we need to bound the expectation of its inverse, hence the following lemma.
We introduce the notation $(z)_+ = \max(0,z)$.
\begin{lemma}[Expected MAP natural parameter]\label{lem:expected-map-natural-parameter-gaussian}
	Define the variable $a = n_0 \frac{\mu_0}{\mu^*}$, which characterizes 
	the importance of the prior relative to the true parameters.
	The expectation of the natural parameter of a MAP of $\Gamma(\alpha,\beta)$ is bounded by
\begin{equation}
		\frac{n_0+n}{a+n}
		=\frac{\mu^*}{\mu_n}
		\leq \expect{\frac{\mu^*}{\MAPm}} 
		= \expect{\frac{\MAPt}{\nat^*}} 
		\leq
		\frac{n_0 +n}{a+(n-\inv{\alpha} )_+} \; , \forall n\geq 0 \; .
		\label{eq:lemma_nat_bound}
	\end{equation}
\end{lemma}

\begin{proof}
	The lower bound can be readily obtained by applying Jensen's inequality 
	to the convex function $x\mapsto\inv{x}$ for $x>0$.
	The upper bound requires more work.
	To start, let us plug in the definition of $\MAPm$
\alignn{
		\expect{\frac{\mu^*}{\MAPm}} 
		= \expect{\frac{\MAPt}{\nat^*}} 
		= \expect{\frac{(n_0+n) \mu^*}{n_0 \mu_0 + \sum_i X_i}}  
		= \expect{\frac{n_0+n}{n_0 \frac{\mu_0}{\mu^*}+ \sum_i \frac{X_i}{ \mu^*} }} 
		= \expect{\frac{n_0+n}{a + \Gamma(n\alpha, \alpha)}} \; ,
}
	where  $\sum_i \frac{X_i^2}{ \mu^*} \sim \Gamma(n\alpha, \alpha)$ is a gamma random variable and $a=n_0 \frac{\mu_0}{\mu^*}$. 
	Further note that 
\alignn{
	\expect{\frac{n_0+n}{a + \Gamma(n\alpha, \alpha)}} 
	= \frac{n_0 + n}{a} \expect{\frac{1}{1 + \Gamma(n \alpha , a \alpha) }} \; .
}
This kind of integrals can be expressed with {generalized exponential integral functions}
	\begin{align}
		E_k(z) = \int_1^\infty \frac{e^{-z t} }{t^k} dt \; ,
	\end{align}
with \href{http://dlmf.nist.gov/8.19.E4}{the formula} \citep[Eq.~8.19.4]{DLMF}
\alignn{
	\expect{\frac{1}{1 + \Gamma(\alpha , \beta) }}
	= \frac{\beta^{\alpha}}{\Gamma(\alpha)} \int_0^\infty \frac{x^{\alpha-1}e^{-\beta x} }{1+x}  dx 
	= \beta e^{\beta} E_\alpha (\beta) \; .
	\label{eq:expect1+gamma}
}
Overall we get
\alignn{
\expect{\frac{\mu^*}{\MAPm}}  = (n_0 +n) \alpha e^{a \alpha} E_{n\alpha} ( a \alpha) 
}
	Now our goal is to bound this generalized exponential integral with simpler functions.
	Fortunately, mathematicians have been working on these integrals for decades.
	For instance , we have \href{https://dlmf.nist.gov/8.19.E21}{the general bound} \citep[Eq.~8.19.21]{DLMF}
	\begin{align}
		e^x E_k(x) 
		&\leq \inv{x + k - 1} , \ \forall k>1\\
		\iff \expect{\frac{\mu^*}{\MAPm}} 
		&\leq \frac{(n_0 +n)\alpha}{a\alpha + n\alpha - 1}
		= \frac{n_0 +n}{a + n - \inv{\alpha}} , \ \forall n >\inv{\alpha} \; .
	\end{align}
We are left with a special case when  $n\leq \alpha$. 
Then we can use the trivial bound 
\begin{align}
	\expect{\inv{a + X^2}} < \inv{a}.
	\label{eq:trivial_bound}
\end{align}
to conclude the proof.
\end{proof}

When $n<\inv{\alpha}$, it is possible to get a much tighter bound by exploiting \href{http://dlmf.nist.gov/8.19.E12}{the recurrence relationship} \cite[Eq.~8.19.12]{DLMF}
\alignn{
	\alpha E_{\alpha+1}(\beta) + \beta E_\alpha(\beta) 
	= e^{-\beta} 
}
and combining it with the inequality \citep[Eq.~8.19.21]{DLMF} to get
\alignn{
	\beta e^\beta E_\alpha(\beta)
 	= 1 - \alpha e^\beta E_{\alpha+1}(\beta)
 	\leq 1 - \frac{\alpha}{\alpha + 1 + \beta}
 	 = \frac{\beta + 1}{\alpha + \beta + 1} \; .
}
Plugging this inequality back into~\eqref{eq:expect1+gamma}, we get the following upper bound for the MAP:
\alignn{
	\expect{\frac{\mu^*}{\MAPm}} 
	\leq \frac{n_0 + n}{a} \frac{a\alpha + 1}{n \alpha + a \alpha +1}
	= \frac{n_0 + n}{a + n +\inv{\alpha}} \paren{1 + \inv{a \alpha}} \; .
}
Unfortunately, this formula does not yield an elegant convergence rate, so we keep it out of the lemma.

\subsection{Proof of MAP Bound}
\label{app:MAP-bound}

We did not find a closed form or an upper bound for the expected logarithm of the MAP $\expect{\log\frac{\MAPm}{\mu^*}}$.
Consequently, we derived a bound for the symmetrized Bregman~\eqref{eq:symmetrized_bregman} instead. 
This bound is asymptotically tight for the Bregman, up to a factor $2$.

There are several ways to write down the convergence rate. The Gaussian variance example can be written in a particularly simple form, so we give  it a special treatment in \S\ref{app:proof-MAP-gaussian}, corresponding to the theorem displayed in the main text.
We make a more general statement about gamma distributions in \S\ref{app:proof-MAP-gamma}.

\subsubsection{Gaussian Variance}
\label{app:proof-MAP-gaussian}

\begin{theorem}[MAP Bound]
The expected symmetrized Bregman~\eqref{eq:symmetrized_bregman} of the MAP of $\cN(  0,\m^*)$ with prior hyper-parameters $(n_0,\m_0)$ is upper bounded as
 \begin{align}
	& \expect{\cS_\conj( \m^*; \hat \m_n^\mathrm{MAP})}
	\leq 
	\left\{\begin{array}{ll}
		\cS_\conj( \mu_*; \mu_0) 					& \text{ if } \ n=0, \\
		\inv{2(n_0+1)}  +  b_1 						& \text{ if }\ n=1,\\
		\frac{1}{n_0 \frac{\m_0}{\m^*} +n-2} + b_n  & \text{ if }\ n\geq 2,
	\end{array}\right.
	\quad 
	\quad 
	\text{where }
	b_n = \frac{(1 + \inv{n_0} - \frac{\m_0}{\m^*})^2}{2 (\frac{\m_0}{\m^*}+\frac{(n-2)_+}{n_0})(1 + \frac{n}{n_0} )} \in O\paren{\frac{n_0^2}{n^2}} \; . 
\end{align}
\end{theorem} 

\begin{proof}
When $n=0$, the inequality is an equality. 
For $n>0$, we expand the symmetrized Bregman~\eqref{eq:symmetrized_bregman} with $\alpha=\half$ to get
\begin{align}
	\expect{\cS_\conj(\mu^*; \MAPm)} 
	&\leq \half \paren{\expect{\frac{\mu^*}{\MAPm}} -1  + \expect{\frac{\MAPm}{\mu^*}} - 1} \; .
\end{align}
The expectation of $\MAPm$ is straightforward
\begin{align}
	 \expect{\frac{\MAPm}{\mu^*}} - 1 
	 = \frac{n_0 \mu_0 + n \mu^*}{(n_0+n)\mu^*} - 1 
	 = \frac{a+n}{n_0+n} - 1 = \frac{a - n_0}{n_0+n}
	 \quadtext{where}
	 a:=n_0\frac{\mu_0}{\mu^*} \; .
\end{align}
There remains the more problematic term with the expectation of the inverse mean parameter, 
for which we use the bound derived in Lemma \ref{lem:expected-map-natural-parameter-gaussian}.

\textbf{When $n\geq 2$}, we get 
\begin{align}
	\expect{\frac{\mu^*}{\MAPm}} - 1 
	\leq \frac{n_0 + n}{a +n - 2} -1
	= \frac{n_0 - a + 2}{a +n - 2}
	 = \frac{2}{a +n - 2} + \frac{n_0 - a}{a +n - 2}
\end{align}
so putting it all together we get 
\begin{align}
	\expect{\bregmanconj(\mu^*; \MAPm)} 
	&\leq \frac{1}{a+n-2} + \frac{n_0-a}{2}  \left( \inv{a+n-2} - \inv{n_0+n} \right)\\
	&= \frac{1}{a+n-2} + \frac{(n_0-a + 1) - 1}{2}  \frac{(n_0 - a + 1) + 1}{(a+n-2)(n_0+n)} \\
	&= \frac{1}{a+n-2} + \frac{(n_0 - a +1)^2 - 1}{2(a+n-2)(n_0+n)} \\
	&\leq \frac{1}{a+n-2} + \frac{(n_0 - a +1)^2}{2(a+n-2)(n_0 + n)} \\
	&= \frac{1}{a+n-2} + \frac{(1 + \inv{n_0} - \frac{\mu_0}{\mu^*})^2}{2 (\frac{\mu_0}{\mu^*}+\frac{n-2}{n_0})(1 + \frac{n}{n_0} )} \; .
\end{align}

\textbf{When $n=1$} the bound~\eqref{eq:lemma_nat_bound}  on the expected natural parameter gives
\begin{align}
	\expect{\frac{\mu^*}{\MAPm}} - 1 
	\leq \frac{n_0 + 1}{a} -1
	= \frac{n_0 +1 - a}{a}
\end{align}
so putting it all together we get
\begin{align}
	2\expect{\bregmanconj(\mu^*; \hat \mu_1)} 
	&\leq \frac{a - n_0 \pm 1}{n_0 + 1}  + \frac{n_0 +1 - a }{a} \\
	& = \inv{n_0+1}  + ( a - n_0 - 1)(\inv{n_0+1} - \inv{a}) \\
	& = \inv{n_0+1}  + \frac{(n_0 + 1 -a)^2}{a(n_0+1)}\\
	&= \inv{n_0+1}  + \frac{(1 + \inv{n_0} - \frac{\mu_0}{\mu^*})^2}{\frac{\mu_0}{\mu^*}(1 + \inv{n_0} )} 
\end{align}
so we  recover the same bias term as when $n\geq2$.
\end{proof}

\subsubsection{Gamma with Known Shape}
\label{app:proof-MAP-gamma}

\begin{theorem}[MAP Bound]
Consider an exponential distribution with sufficient statistics coming from a gamma distribution $\Gamma(\alpha, \beta^*)$ with mean parameter $\mu^* = \frac{\alpha}{\beta^*}$.
The expected symmetrized Bregman~\eqref{eq:symmetrized_bregman} of the MAP  with prior hyper-parameters $(n_0,\mu_0)$ is upper bounded as
\alignn{
    \forall n\geq \inv{\alpha},\quad \quad
	\expect{\cS_\conj(\mu^*, \hat \mu_n)}
	\leq\frac{1}{n_0+n} + 
	\frac{\alpha\paren{\frac{\mu_0}{\mu^*} - \inv{\alpha n_0} - 1 }^2 }{\paren{1+\frac{n}{n_0}}\paren{\frac{\mu_0}{\mu^*} +  \frac{n - \inv{\alpha}}{n_0}}}
}
\end{theorem} 
Note that the second term vanishes when $\mu_0 = \mu^* \paren{1 + \inv{\alpha n_0}}$.
Expressions for $n\alpha < 1$ are less elegant as shown below:
\alignn{
 	\expect{\cS_\conj(\mu^*, \hat \mu_n)}
	\leq \frac{n_0 + n}{a + n +\inv{\alpha}} \paren{1 + \inv{a \alpha}} -1 + \frac{a - n_0}{n_0+n}
	\quadtext{where}
	 a:=n_0\frac{\mu_0}{\mu^*} \; .
}

\begin{proof}
We expand the symmetrized Bregman~\eqref{eq:symmetrized_bregman} to get
\begin{align}
	\expect{\cS_\conj(\mu^*; \MAPm)} 
	&\leq \alpha \paren{\expect{\frac{\mu^*}{\MAPm}} -1  + \expect{\frac{\MAPm}{\mu^*}} - 1} \; .
\end{align}
The expectation of $\MAPm$ is straightforward
\begin{align}
	 \expect{\frac{\MAPm}{\mu^*}} - 1 
	 = \frac{n_0 \mu_0 + n \mu^*}{(n_0+n)\mu^*} - 1 
	 = \frac{a+n}{n_0+n} - 1 = \frac{a - n_0}{n_0+n}
	 \quadtext{where}
	 a:=n_0\frac{\mu_0}{\mu^*} \; .
\end{align}
There remains the more problematic term with the expectation of the inverse mean parameter, 
for which we use the bound derived in Lemma \ref{lem:expected-map-natural-parameter-gaussian}  when $n\alpha \geq 1$
\begin{align}
	\expect{\frac{\mu^*}{\MAPm}} - 1 
	\leq \frac{n_0 + n}{a +n - \inv{\alpha}} -1
	= \frac{n_0 - a + \inv{\alpha}}{a +n - \inv{\alpha}}
\end{align}
so putting it all together we get 
\begin{align}
	\expect{\bregmanconj(\mu^*; \MAPm)} 
	&\leq 
	\alpha \paren{\frac{n_0 - a + \inv{\alpha}}{a +n - \inv{\alpha}} 
	 + \frac{a - n_0 \pm \inv{\alpha}}{n_0+n}} \\
	 &=
	 \alpha \paren{n_0 + \inv{\alpha} -a}  
	 \left( \inv{a+n-\inv{\alpha}} - \inv{n_0+n} \right)
	 + \inv{n_0 + n} \\
	 &= \inv{n_0 + n} + \alpha \frac{(n_0 + \inv{\alpha} - a)^2}{(n_0 + n)(a+ n - \inv{\alpha} )} \\
	&= \inv{n_0 + n} + \alpha \frac{(1 + \inv{\alpha n_0} - \frac{\mu_0}{\mu^*})^2}{ (1 + \frac{n}{n_0} ) (\frac{\mu_0}{\mu^*} - \inv{\alpha n_0} +\frac{n}{n_0})} \; .
\end{align}
\end{proof}

\subsection{On the Choice of a Prior}
\label{app:prior-choice}
The optimal $\mu_0$ is larger than $\mu^*$ for small $n_0$ and small $n$.
Indeed, the upper bound~\eqref{eq:MAP_rate} has a bias term that is $0$ when $\frac{mu_0}{\mu^*} = 1 + \inv{n_0}$, e.g. for large values of $n_0$, it is $\mu_0=\mu^*$ is the best prior, but for small $n_0$, one better sets larger values for $\mu_0$. In Figure~\ref{fig:optimal_n0}, we observe this behavior numerically.

\begin{figure}[ht]
	\centering
	\includegraphics[width=\textwidth]{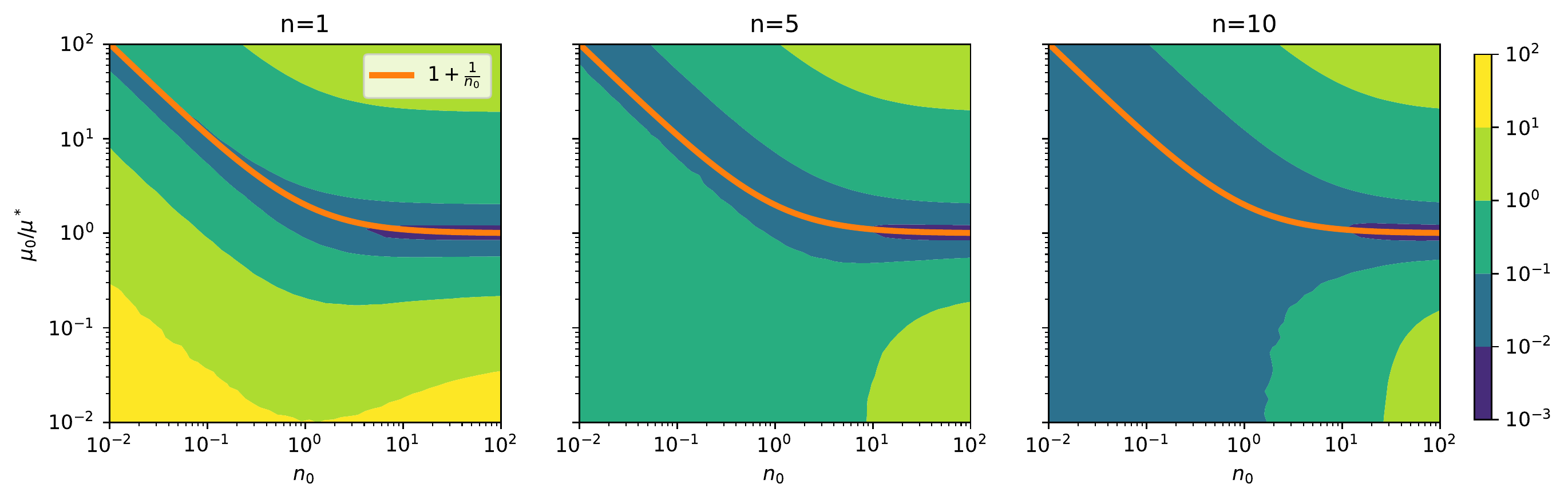}
	\caption{Contours of $\expect{\KL(\nat^*, \hat \nat_n)}$  for the Gaussian variance with $n\in\{1,5,10\}$, $\mu^*=1$ and $n_0, \mu_0$ spanning $[10^{-2}, 10^2]$. The expectation was estimated with $10^4$ draws for each value of $n_0, \mu_0$ We observe that the line $\frac{\mu_0}{\mu^*} = 1 + \inv{n_0}$ coincides with the bottom valley of this landscape.
	}
	\label{fig:optimal_n0}
\end{figure}

\section{COMPLEMENTS ON GAUSSIANS}
\label{app:gaussian}

In this section, we prove that for a Gaussian, the entropy and log-partition functions are self-concordant.
We also provide complementary illustrations of 
these functions in \cref{fig:logpart-entropy}, 
their gradients (e.g. the mirror maps) in \cref{fig:mirrormaps},
and paths taken by MLE and MAP in \cref{fig:gaussian-paths} .

The Gaussian log-partition function and entropy are, up to constants,
\alignn{
	\logpart(\nat) &= \frac{\nat_1^2}{-4\nat_2} - \half \log(-\nat_2) \label{eq:app-logpart} \\
	\conj(\m) &= - \half \log (\mu_2 - \mu_1^2) \; , \label{eq:app-entropy}
}
where $\nat_1\in\real, \nat_2 <0$ and $\mu_1\in\real, \mu_2>\mu_1^2$.
We provide definitions of self-concordance in \cref{app:self-concordant}.

\begin{figure}[ht]
	\centering
	\includegraphics[width=.45\textwidth]{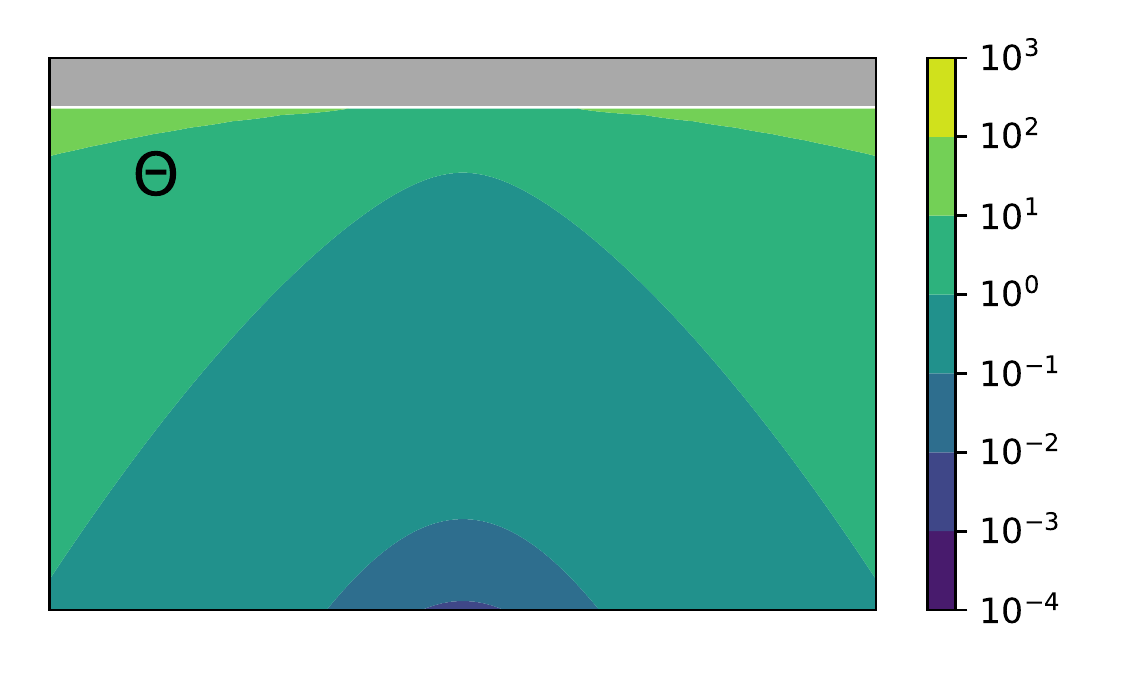}
	\includegraphics[width=.45\textwidth]{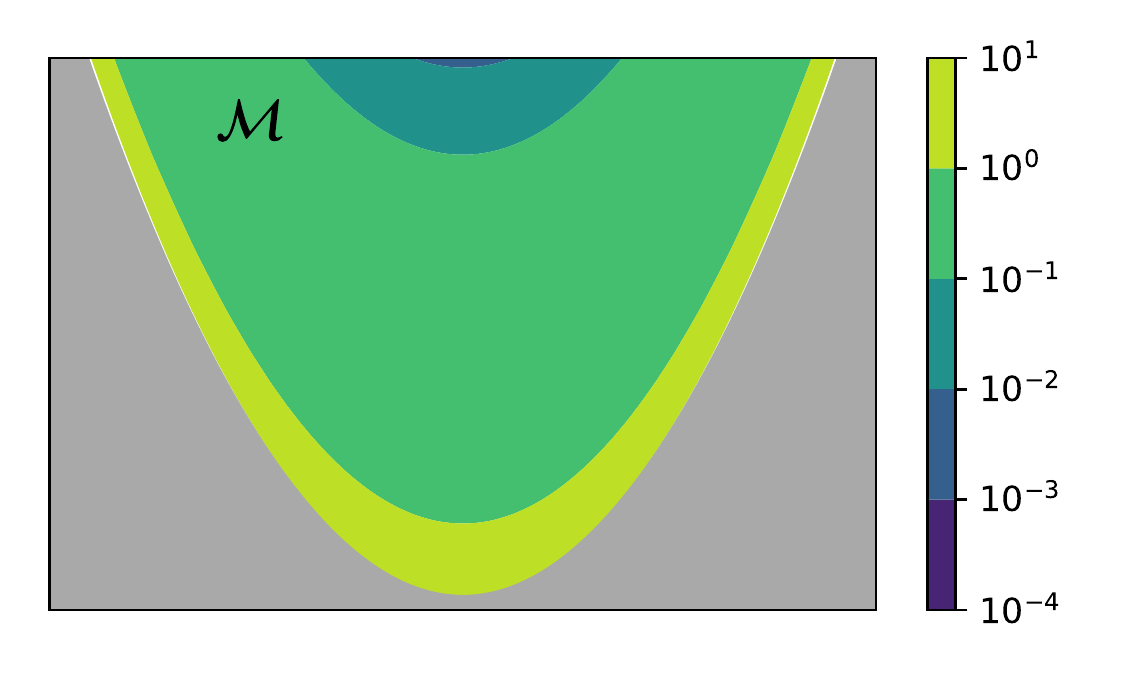}
	\caption{
	(left) Contours of the log-partition function~\eqref{eq:app-logpart}.
	(right) Contours of the entropy~\eqref{eq:app-entropy}.
	}
	\label{fig:logpart-entropy}
\end{figure}

\begin{figure}[ht]
	\centering
	\includegraphics[width=.9\textwidth]{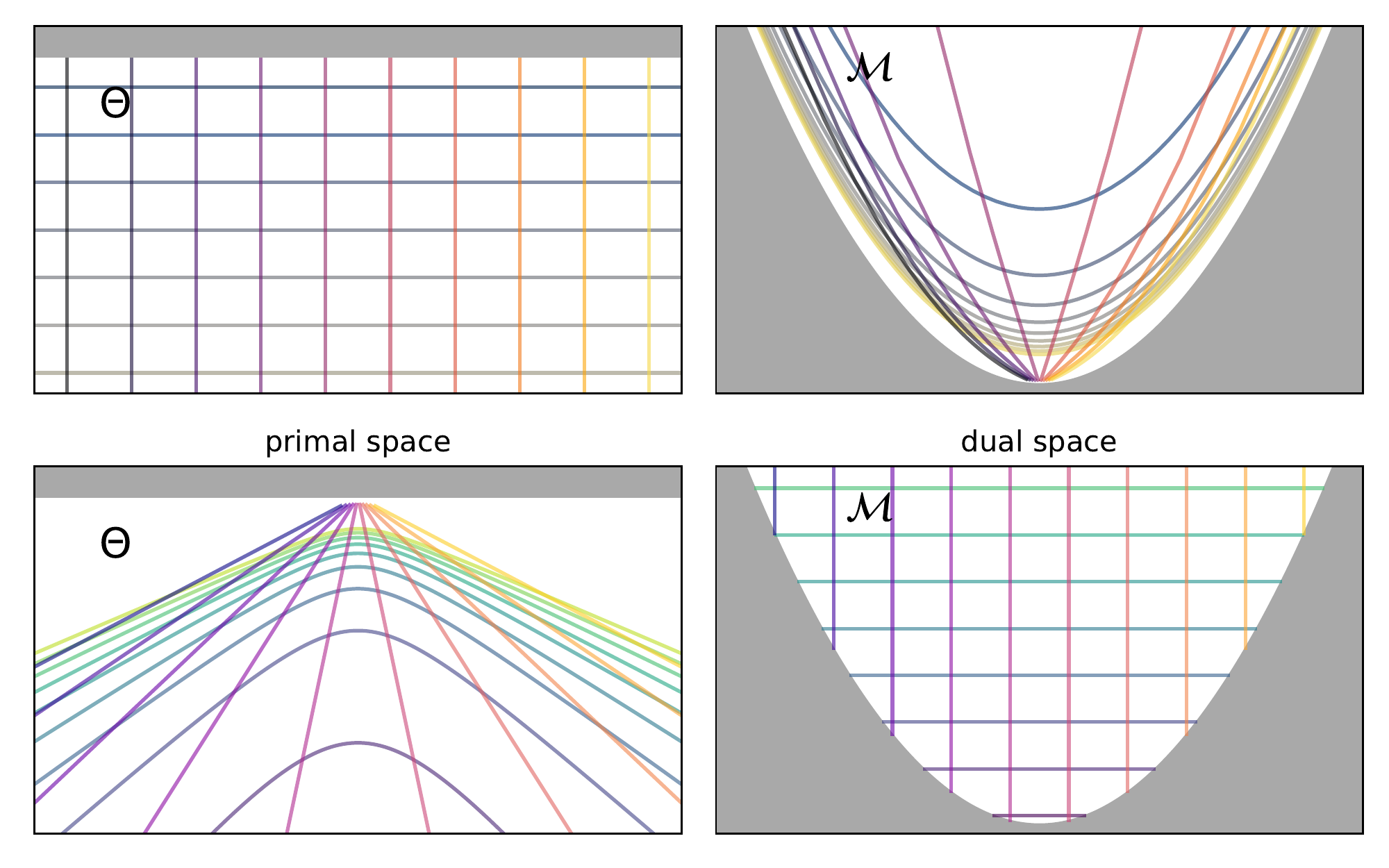}
	\caption{
	(top) Grid deformation produce by $\nabla \logpart(\nat)$.
	(bottom) Grid deformation produced by $\nabla\conj(\m)$.
	}
	\label{fig:mirrormaps}
\end{figure}

\begin{figure}[ht]
	\centering
	\includegraphics[width=.9\textwidth]{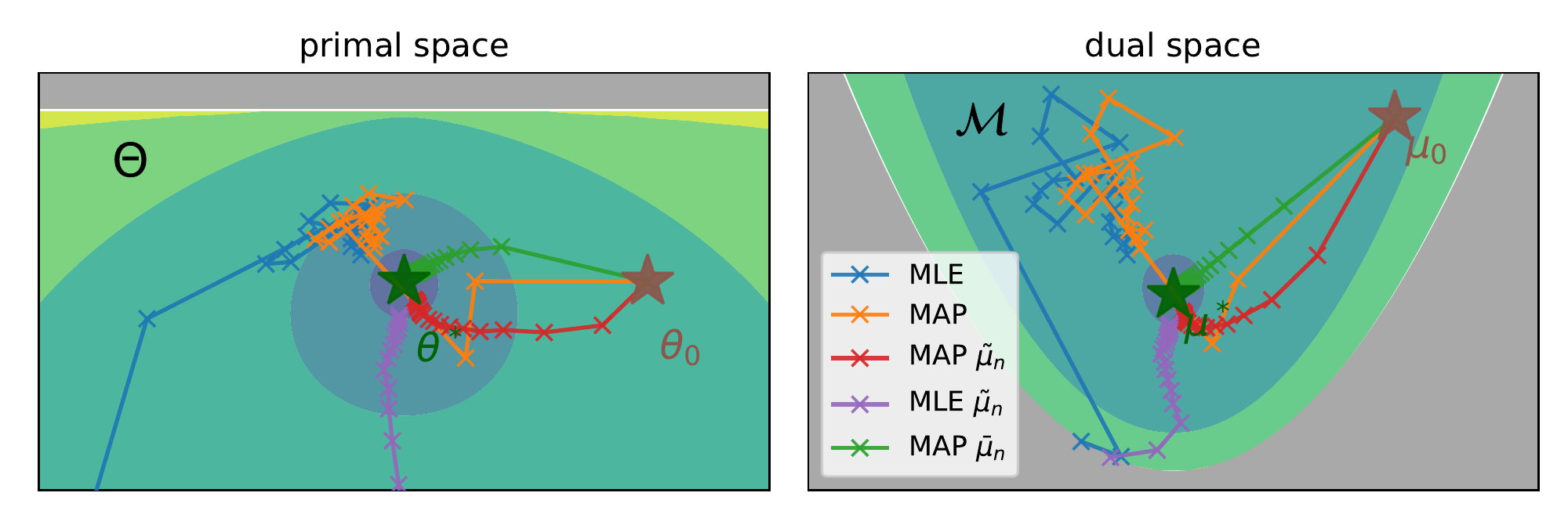}
	\caption{
		Paths taken by MLE (blue) and MAP (orange) on top of contours for $\KL(\nat^* || \nat)$. 
		We set $\mu^*=(0, 1), \mu_0 = (1, 2)$, and $n_0=4$, and $n$ varies from $1$ to $20$.
		In green, red and purple, we represent the paths respectively taken by the MAP dual expectation  $(\bar\nat_n, \bar \mu_n)$, MAP primal expectation $(\tilde\nat_n,\tilde\mu_n)$, and MLE primal expectation. Recall that the MLE dual expectation is $\mu^*$ itself.
	}
	\label{fig:gaussian-paths}
\end{figure}

\paragraph{Entropy is Self-Concordant.}
\citet[Example 4.1.1.4,   p.177]{nesterov2003introductory} proves that logarithmic barriers for second order regions are self-concordant, 
that is functions of the form 
\alignn{
	f(\theta) = -\log\paren{\alpha + \lin{a, \theta} - \frac{1}{2}\lin{\mA \theta, \theta}}
	\text{ on }
	\left\{\theta\in\mathbb{R}^n \cond \alpha + \lin{a, \theta} - \frac{1}{2}\lin{\mA \theta, \theta} > 0\right\}.
}
The entropy~\eqref{eq:app-entropy} fits into this definition with $\mA = \begin{pmatrix} 2 & 0 \\ 0 & 0 \end{pmatrix}$ and $a = ( 0 \ 1 )^\top$.

\paragraph{Log-partition is Self-Concordant.}
As proved in \citet{nesterov1994interior}, self-concordance is preserved by Fenchel conjugacy.
Since $\conj$ is self-concordant, $\logpart$ is as well.
For a more accessible reference, see also \citet[Prop.~6]{sun2019generalized}.

\section{ASYMPTOTIC DERIVATION}
\label{app:asymptote}
In this section we fill-in the lines of \S\ref{ssec:asymptote} to prove Equation~\eqref{eq:asymptote}.
Approximating $\conj$ with a second order Taylor expansion yields
\aligns{
    \bregmanconj(\m^* ; \m)
    &= \frac{1}{2}\norm{\m - \m^*}^2_{\mF}
    + O(\norm{\m - \m^*}^3),
}
where the norm is induced by the matrix
\aligns{
    \mF
    :=\nabla^2\conj(\m^*)
    = \nabla^2\logpart(\nat^*)^{-1}
    = \Cov_{\nat^*}(T)^{-1},
}
where the second equality is a general property of convex conjugates.
Plugging the MLE~\eqref{eq:defMLE} into this quadratic and expanding it yields
\aligns{
	\E \half \norm{\inv{n}  \sum_i T_i - \m^* }_\mF^2 
	&=\inv{2n^2} \sum_i \E \norm{T_i - \m^* }_\mF^2 
	+ \inv{2n^2} \sum_{i\neq j}\expect{T_i - \m^*}^\top \mF \overbrace{\expect{T_j - \m^*}}^{0}\\
	&=\inv{2n}  \E \norm{T_1 - \m^* }_\mF^2 \\
	& = \inv{2n} \Tr(\mF\ \expect{(T_1 - \m^*) (T_1 - \m^*)^\top}) \\
	&= \inv{2 n} \Tr(\mF \Cov_{\nat^*}(T)) = \frac{d}{2n} ,
}
where on the first line we used independence of samples, on the second line we used the fact that samples are identically distributed, and on the third line we used the  trace trick along with the linearity of the trace $\Tr$.
On the way, this also proves that for the MLE $\norm{\m - \m^*}^3 \in O(n^{-\half[3]})$.
This yield the final rate for the MLE
\begin{align}
	\E \bregmanconj \paren{\E [T(X)] ; \inv{n}  \sum_i T_i}
	= \frac{d}{2n} + O(n^{- \frac{3}{2}} ) \; .
\end{align}
For the MAP~\eqref{eq:defMAP}, the quadratic decomposes into bias and variance:
\begin{align}
	\E \half \norm{\mu^* -  \frac{n_0 \mu_0 + \smallsum_i T(x_i)}{n_0+n} }^2_{\mF}
	&= \frac{n d}{2(n+n_0)^2}  +  \frac{n_0^2}{(n+n_0)^2} \half \norm{\mu^* -  \mu_0}^2_{\mF} \\
	&= \frac{d}{2n} + O\left(\frac{1 + \norm{\mu^* -  \mu_0}^2_{\mF} }{n^2} \right) \; .
\end{align}
This $O(n^{-2})$ term is dominated by the $O(n^{-\half[3]})$ term from the quadratic approximation of the Bregman, yielding the same first order rate as for the MLE
\begin{align}
	\E \bregmanconj \paren{\E [T(X)] ; \frac{n_0 \mu_0 + \smallsum_i T(x_i)}{n_0+n} }
	= \frac{d}{2n} + O(n^{- \frac{3}{2}} ) \; .
\end{align}

\section{SELF-CONCORDANCE}
\label{app:self-concordant}

In this section, we define self-concordance and we prove Proposition~\ref{prop:selfConcordant}.

\begin{definition}[Self-concordance]
\label{def:self-concordance}
A convex function is $F:\real^p \rightarrow \real$ is self-concordant if it is differentiable $3$ times and if for all $w, v \in\real^p$ the function $g(t) = F(w+tv)$ satisfies for all feasible $t$
\alignn{
	\abs{g'''(t)} \leq 2 g''(t)^{\frac{3}{2}} \; .
}
\end{definition}


{\bf Clarification:} In the main text in Section~\ref{ssec:local-quadratic}, we claimed that the Fenchel conjugate of a 1-dimensional function is also self-concordant. 
Actually, this is also true in higher dimensions, as proved by \citet{nesterov1994interior}. See \citet[Prop.~6]{sun2019generalized} for a more accessible reference.

\subsection{Properties of Self-concordant functions}

We quickly review some important properties of self-concordant functions, introduced in \citep{nesterov2003introductory}. We start with a some notation. Let $A^*$ be a self-concordant function. Then, we write
\begin{itemize}
	\item the local norm $\|\cdot\|_x = \sqrt{ \langle \nabla^2A^*(x)\cdot,\, \cdot \rangle }$ 
	\item the distance function $\omega(t) = t-\ln (1+t),\;t\geq 0$, and its dual $\;\omega^*(t) = -t-\ln(1-t)$ defined for $t\in [0,1]$.
\end{itemize}
Note that $\omega^*(t)$ is positive, convex and monotonically increasing for $t\in[0,1]$.

\begin{figure}[ht]
	\centering
	\includegraphics[width=0.4\textwidth]{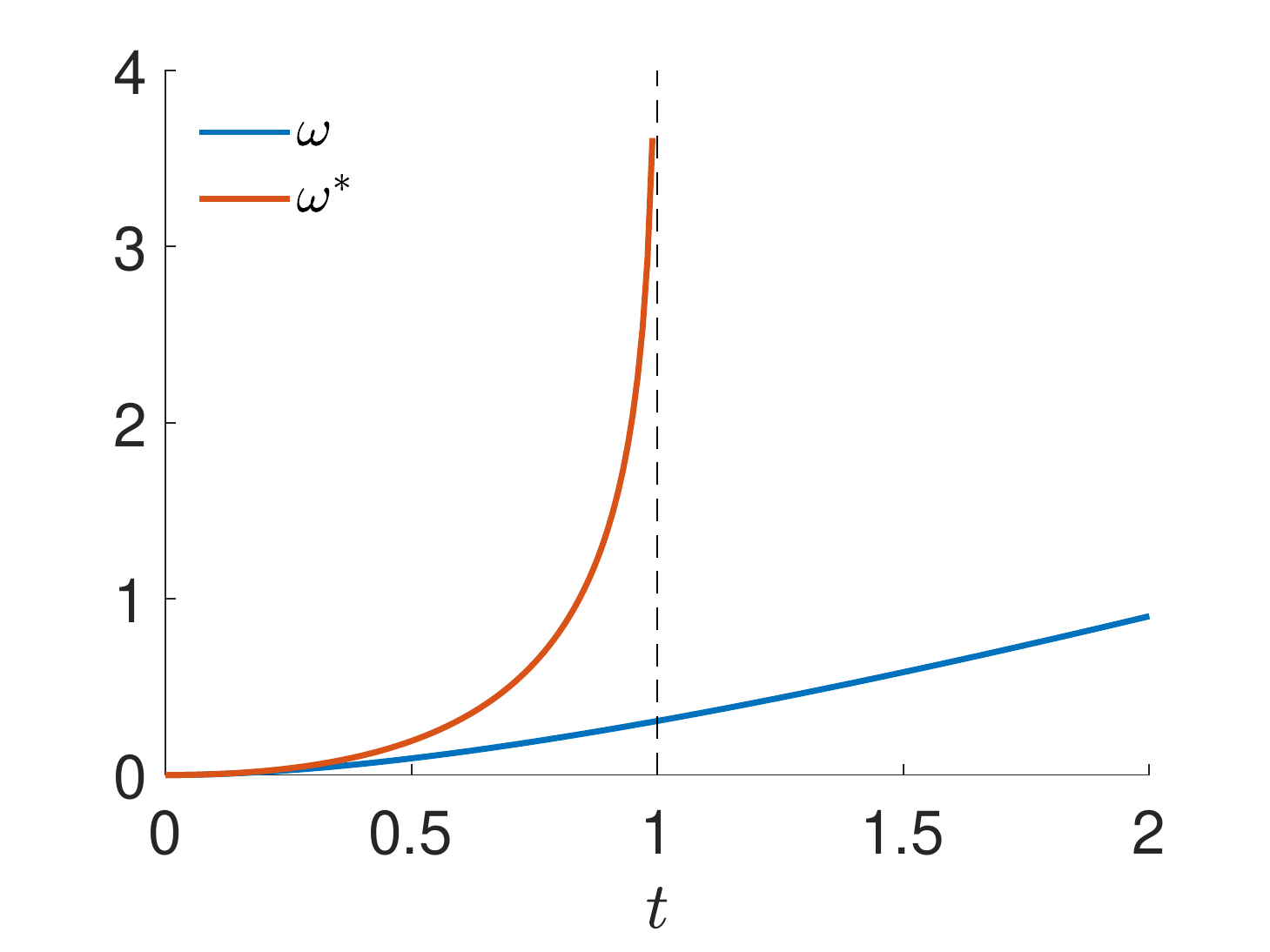}
	\includegraphics[width=0.4\textwidth]{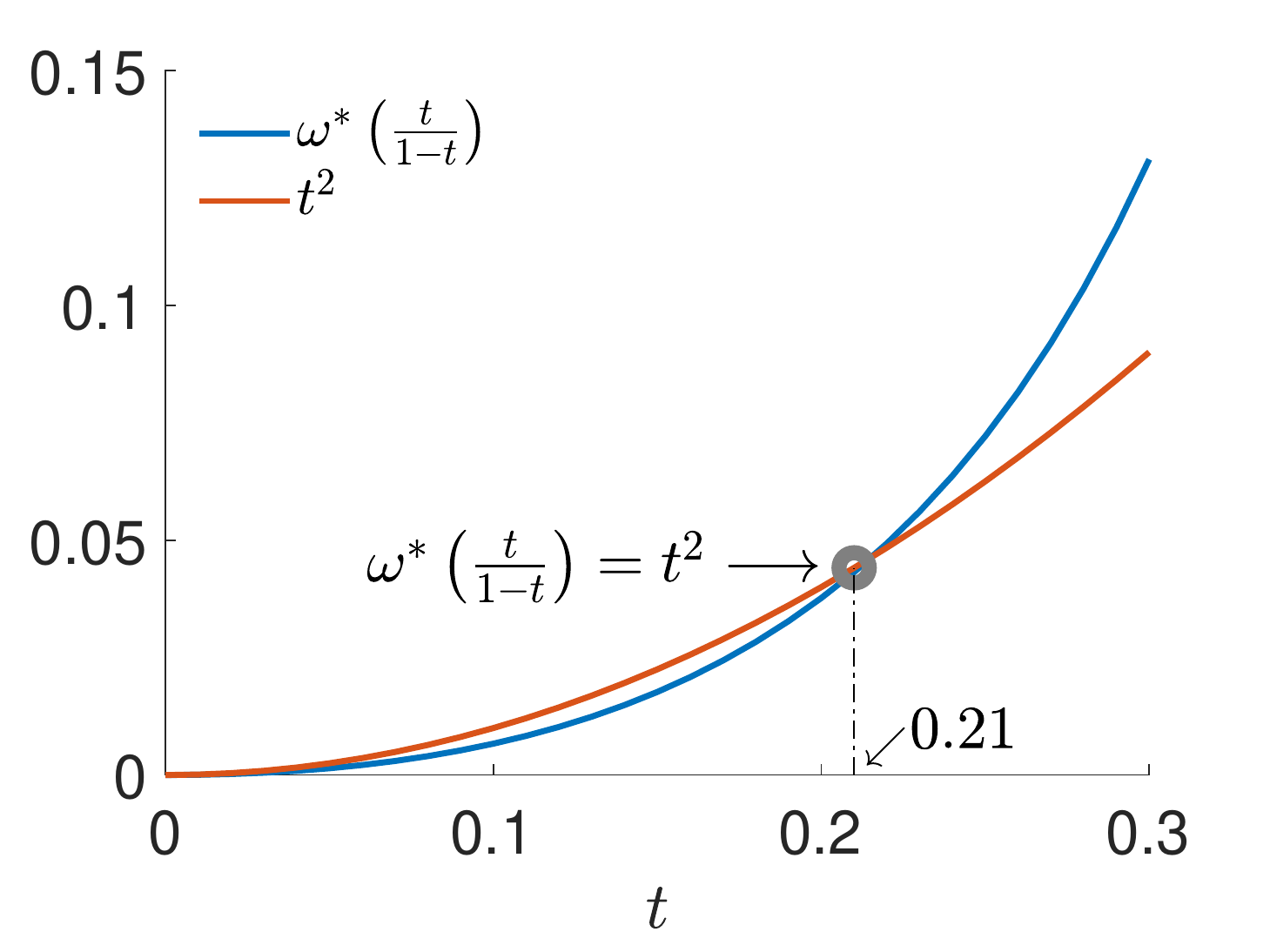}
	\caption{(left) Graph of the distance function $\omega$ and its dual $\omega^*$. (right) Graph of the dual of distance function $\omega^*$ evaluated at $\frac{t}{1-t}$, compared with $t^2$. The two curves cross each other at $t\approx 0.21$.}
	\label{fig:omega}
\end{figure}

We now present two important results. The first one shows how to convert the local norm $\|y-x\|_y$ using $\|y-x\|_x$.
\begin{proposition}\label{prop:conversion_norm}
	(Conversion of norms, \citep[Theorem 4.1.5]{nesterov2003introductory}) For any $x,\,y\in\dom A^*$, if $\|y-x\|_x<1$, then
	\[
		\|y-x\|_y \leq \frac{\|y-x\|_x}{1-\|y-x\|_x}.
	\]
\end{proposition}

The next result shows that, if $y$ is sufficiently close to $x$, then we can bound the Bregman divergence of $A^*$ using the distance function $\omega^*$ and local norms.
\begin{proposition}(Upper bound of self-concordant functions \citep[Theorem 4.1.8]{nesterov2003introductory}) \label{prop:upper_bound_self_concordance}
	For any $x,\,y\in\dom A^*$, if $\|y-x\|_x<1$, then
	\[
		\bregmanconj(y,x) \leq \omega^*(\|y-x\|_x).
	\]
\end{proposition}

We are now ready to prove Proposition~\ref{prop:selfConcordant}.

\subsection{Proof of Proposition~\ref{prop:selfConcordant}.}

We start with Proposition~\ref{prop:upper_bound_self_concordance}, evaluated at $y=\meanp^*$ and $x=\meanp$:
\[
	\bregmanconj(\meanp^*,\meanp) \leq \omega^*(\|\meanp^*-\meanp\|_{\meanp}).
\]
This hold if $\|\meanp^*-\meanp\|_{\meanp}<1$. Since $\omega^*$ is monotonically increasing, we can replace $\|\meanp^*-\meanp\|_{\meanp}$ by its upper bound from Proposition~\ref{prop:conversion_norm},
\[
	\omega^*(\|\meanp^*-\meanp\|_{\meanp}) \leq \omega^*\left(\frac{\|\meanp^*-\meanp\|_{\meanp^*}}{1-\|\meanp^*-\meanp\|_{\meanp^*}}\right),
\]
under the conditions that $\|\meanp^*-\meanp\|_{\meanp^*}<1$ (to satisfy the assumption of Proposition~\ref{prop:conversion_norm}) and $\frac{\|\meanp^*-\meanp\|_{\meanp^*}}{1-\|\meanp^*-\meanp\|_{\meanp^*}} <1$ (to ensure that $\|\meanp^*-\meanp\|_{\meanp} <1$). Those two conditions holds if $\|\meanp^*-\meanp\|_{\meanp^*} <0.5$. Now, we use the bound (see figure \ref{fig:omega})
\[
	\omega^*\left(\frac{t}{1-t}\right) \leq t^2, \quad 0\leq t \leq 0.21,
\]
and replace $t$ by $\|\meanp^*-\meanp\|_{\meanp^*}$. This finally gives the sequence of inequalities
\[
	\bregmanconj(\meanp^*,\meanp) \leq \omega^*(\|\meanp^*-\meanp\|_{\meanp}) \leq \omega^*\left(\frac{\|\meanp^*-\meanp\|_{\meanp^*}}{1-\|\meanp^*-\meanp\|_{\meanp^*}}\right) \leq \|\meanp^*-\meanp\|_{\meanp^*}^2,
\]
that holds while $ \|\meanp^*-\meanp\|_{\meanp^*}<0.21$, which is the desired result.

\section{BIAS-VARIANCE}
\label{app:bias-variance}
In this section, we start from the notions of bias and variance introduced in \cref{eq:bias-variance}.
First, we prove that the bias of the MLE of a Gaussian variance decreases in $O(\inv{n^2})$. 
Then we prove that assuming~\eqref{eq:hanzely} holds, whether uniformly or in expectation, yields a convergence rate on the variance term.

\subsection{Bias of a Gaussian Variance MLE}
For a Gaussian variance model, the MLE follows a scaled $\chi^2(n)$ distribution. 
This means that
\alignn{
	\frac{\m^*}{\tilde \m_n} 
	= \frac{\tilde\nat_n}{\nat^*}
	=\expect{\frac{\hat\nat_n}{\nat^*}}
	=\expect{\frac{\m^*}{\hat \m_n} }
	=\expect{\frac{n}{\chi^2(n)}}
	= \frac{n}{n-2}\; .
}
Consequently,
\alignn{
	\bregmanconj(\m^*, \tilde \m_n) 
	&= \half \paren{\frac{n}{n-2} -1 + \log\frac{n-2}{n}}\\
	&= \frac{1}{n-2} + \half \log\paren{1 - \frac{2}{n}}\\
	&\leq \inv{n-2} - \inv{n} 
	= \frac{2}{n(n-2)},
}
so the bias of the MLE of a Gaussian Variance decreases like $O\paren{\inv{n^2}}$.

\subsection{Expectation of SMD's Variance Assumption}
The first step is to notice the symmetrized Bregman can be expressed as an inner product between primal and dual parameters
\alignn{
	\cS_\conj (\mu , \bar\mu)
	= \lin{ \nabla \conj(\mu) - \nabla \conj(\bar\mu) , \mu - \bar\mu}
	= \lin{ \nat -\bar\nat , \mu - \bar\mu} \; .
}
Now notice that~\eqref{eq:hanzely} features $\mu := \hat \mu_{n+1} = \hat\m_{n} - \lr g(\hat \nat_{n})$ stochastic and $\bar\mu := \bar\mu_{n+1} =  \hat\m_{n}   -\lr \nabla f(\hat \nat_{n})$ deterministic such that $\E_g[\mu] = \bar\mu$.
For such a pair of variables, the expectation of the symmetrized Bregman corresponds to a covariance between primal and dual parameters
\alignn{
	\expect{\cS_\conj (\mu , \E[\mu])} 
	= \expect{\lin{\nat, \mu - \E[\mu]}}
	= \underbrace{\expect{\lin{\nat, \mu}} -\lin{\expect{\nat}, \expect{\mu} } }_{\Cov(\nat, \mu)}
	=\expect{\lin{\nat - \E[\nat],\mu}} 
	= \expect{\cS_\logpart(\nat, \E[\nat])} \; .
}
The last equality holds by symmetry between the roles of $\logpart$ and $\conj$.
Note that the middle covariance formulation is actually the one  used by \citet{hanzely2018fastest}.
Now, \cref{eq:bias-variance} defines the variance as 
\alignn{
	\expect[1:n]{\bregmanconj(\tilde \m_n , \hat\m_n) }
	=\expect[1:n]{\bregman(\hat\nat_n , \E_{1:n}[\hat\nat_n]) }
	\leq \expect[1:n]{\cS_A(\hat\nat_n , \E_{1:n}[\hat\nat_n]) }
	=  \expect[1:n]{\cS_\conj(\hat\mu_n , \E_{1:n}[\hat\mu_n]) }\;,
	\label{eq:bregman-covariance}
}
where expectations $\E_{1:n}$ are on all samples $X_1, \dots, X_n$, whereas~\eqref{eq:hanzely} is written with 
\alignn{
	\expect[n]{\cS_\conj(\hat\mu_n , \E_{n}[\hat\mu_n]) } \leq \lr^2 C \; ,
}
where the expectation $\E_{n}$ is taken over only the last sample $X_n$, and the bound should hold uniformly over all $\hat\mu_{n-1}$.
Taking the expectation over $\hat\mu_{n-1}$ instead gives
\alignn{
	\expect[1:n]{\cS_\conj(\hat\mu_n , \E_{n}[\hat\mu_n]) }  \leq \lr^2 C \;.
}
The only difference with the RHS of \cref{eq:bregman-covariance} is in the inner expectation.
To overcome this difference, we need to plug in the form of $\hat \mu_n = \frac{n_0\mu_0 + \sum_i T_i}{n_0 +n}$. Notice that 
\alignn{
	&\MAPm - \E_n[\MAPm] = \frac{T_n - \m^*}{n_0+n} \\
	\implies & \expect[1:n]{\cS_\conj(\hat\mu_n , \E_{n}[\hat\mu_n]) }
	= \inv{n_0+n} \expect[1:n]{ \lin{\MAPt , T_1 - \m^*} },
}
while
\alignn{
	&\MAPm - \E_{1:n}[\MAPm] = \frac{\smallsum_i (T_i - \m^*)}{n_0+n} \\
	\implies & \expect[1:n]{\cS_\conj(\hat\mu_n , \E_{1:n}[\hat\mu_n]) }
	=\frac{n}{n_0+n} \expect[1:n]{ \lin{\MAPt , T_1 - \m^*} } \; .
}
In the end, we get that the variance is dominated by $n$ times the expectation of \cref{eq:hanzely}:
\alignn{
	\expect[1:n]{\bregmanconj(\tilde \m_n , \hat\m_n) }
	\leq n \expect[1:n]{\cS_\conj(\hat\mu_n , \E_{n}[\hat\mu_n]) }  \; .	
}
If assumption~\eqref{eq:hanzely} holds, then we have
\alignn{
\expect[1:n]{\bregmanconj(\tilde \m_n , \hat\m_n) }
\leq n \lr_n^2 C \in O\paren{\inv{n}} \; ,
}
where we assumed $\lr_n  \in O\paren{\inv{n}}$.
In conclusion, assuming~\eqref{eq:hanzely} holds uniformly or in expectation immediately implies a $O\paren{\inv{n}}$ convergence rate on the variance.

\section{REVIEW OF SMD}
\label{app:SMD}

We use this section to give more details on the (stochastic) mirror descent algorithm.
We start with gradient descent with step-sizes $\gamma$. 
the update $\theta_{n+1} = \theta_n - \gamma \nabla f(\theta_n)$ can be viewed as the minimization of the linear approximation of $f$ at $\theta_n$
$f(\theta) \approx f(\theta_n) + \lin{\nabla f(\theta_n), \theta - \theta_n}$, 
alongside with quadratic penalty scaled by ${1}/{\gamma}$:
\alignn{
	\theta_{n+1} = \arg\min_\theta f(\theta_n) + \lin{\nabla f(\theta_n), \theta - \theta_n} + \frac{1}{\gamma} \frac{1}{2}\norm{\theta - \theta_n}^2.
}
Mirror descent generalizes the above, using the Bregman divergence induced by a (Legendre) function $A$ instead of the Euclidean norm as follow, 
\alignn{\label{eq:mirror-descent-primal}
	\theta_{n+1} = \arg\min_\theta f(\theta_n) + \lin{\nabla f(\theta_n), \theta - \theta_n} + \frac{1}{\gamma} \bregman(\theta, \theta_n).
}
Mirror descent coincides with gradient descent if $A(\theta) = \frac{1}{2}\norm{\theta}^2$.
As \cref{eq:mirror-descent-primal} is convex, the minimum is at a stationary point, 
found by taking the derivative and setting to 0, 
leading to the update $\theta_{n+1}$ satisfying
\alignn{
	\nabla f(\theta_n) + \frac{1}{\gamma}\paren{\nabla A(\theta_{n+1}) - \nabla A(\theta_n)} = 0
	&&\implies&&
	\nabla A(\theta_{n+1}) = \nabla A(\theta_n) - \gamma \nabla f(\theta_n).
}
Expressed with the dual parameters, we obtain $\mu_{n} = \nabla A(\theta_{n})$, 
$\mu_{n+1} = \mu_n - \gamma \nabla f(\theta_n)$. 

{\bf In our case,} where the objective function is the (negative) log-likelihood of an exponential family, we have
\alignn{
	f(\theta) = A(\theta) - \lin{\frac{1}{n} \sum_{i=1}^n T(X_i), \theta}.
}
Using Mirror descent with a step-size of $1$ and the log-partition function $A$ as the reference function gives
\alignn{\begin{aligned}
	\mu_{n+1} 
	= \mu_n - \nabla f(\theta_n)
	= \mu_n - \paren{\nabla A(\theta_n) - \frac{1}{n}\sum_{i=1}^n T(x_i)}
	= \frac{1}{n}\sum_{i=1}^n T(x_i).
	\end{aligned}
}
In a stochastic, online version 
where the linearization of the objective is obtained from iid samples, 
a decreasing step-size of $\gamma_n = 1/n$ recovers the ``online'' estimate of the MLE.
The case of $\mu_1 = T(x_1)$ follows from the above, 
and in general, assuming it holds for $\mu_n$, 
\alignn{\begin{aligned}
	\mu_{n+1} &= \mu_n - \gamma_n g(\theta_n)
	= \mu_n - \gamma_n (\mu_n - x_n)
	\\
	&= (1 - \frac{1}{n}) \mu_n + \frac{1}{n} T(x_n)
	=
	\frac{n-1}{n} \frac{1}{n-1}\sum_{i=1}^{n-1} T(x_i)
	= 
	\frac{1}{n} \sum_{i=1}^{n} T(x_i)
	+ \frac{1}{n} T(x_n).
\end{aligned}}
The derivation in the main text gives the more general result, 
of using step-sizes of the form $1/(n+n_0)$ to recover online MAP estimation 
with a conjugate prior depending on $n_0$ and the initial estimate of the parameters $\theta_0$.

 \end{document}